\documentclass{article} 
\usepackage{iclr2027_conference,times}

\usepackage{amsmath,amsfonts,bm}

\def\eqref#1{equation~\ref{#1}}

\def\1{\bm{1}}

\DeclareMathAlphabet{\mathsfit}{\encodingdefault}{\sfdefault}{m}{sl}
\SetMathAlphabet{\mathsfit}{bold}{\encodingdefault}{\sfdefault}{bx}{n}

\usepackage{hyperref}
\usepackage{url}

\usepackage[table]{xcolor}
\usepackage{booktabs} 
\usepackage{multirow} 
\usepackage[table]{xcolor} 
\usepackage{graphicx} 
\usepackage{subcaption}
\usepackage{booktabs}
\usepackage[table]{xcolor}
\usepackage[table,dvipsnames]{xcolor}
\usepackage{booktabs}
\usepackage{multirow}

\usepackage{amsthm}
\usepackage{enumitem}
\usepackage{comment}

\iclrfinalcopy

\newtheoremstyle{nospace}
  {3pt}          
  {-2pt}          
  {\itshape}     
  {}             
  {\bfseries}    
  {.}            
  { }            
  {}             
\theoremstyle{nospace}
\newtheorem{lemma}{Lemma}
\newtheorem{proposition}{Proposition}
\newtheorem{definition}{Definition}
\newtheorem{theorem}{Theorem}
\newtheorem{corollary}{Corollary}

\usepackage{titlesec}
\titlespacing*{\section}{0pt}{0.6ex}{0.4ex}
\titlespacing*{\subsection}{0pt}{0.3ex}{0.2ex}

\AtBeginDocument{
\setlength\abovedisplayskip{0.1ex}
\setlength\belowdisplayskip{0.1ex}}
\allowdisplaybreaks[4]

\title{GIFT: Gradient Informed Filtering for Distillation Augmented Reasoning RL}

\author{%
\parbox[t]{\dimexpr\textwidth-2\tabcolsep\relax}{%
\centering
\normalfont
\makebox[\linewidth][c]{%
Shijie Li\textsuperscript{1, 2},
Yilin Gao\textsuperscript{2},
Tieyuan Chen\textsuperscript{1},
Yun Luo\textsuperscript{2},
Weiyao Lin\textsuperscript{1,*},
Hang Yu\textsuperscript{2,*}%
}\\[6pt]
\textsuperscript{1}\,Shanghai Jiao Tong University
\quad
\textsuperscript{2}\,Tencent\\[3pt]
{\small shijieli@sjtu.edu.cn}\\[2pt]
{\small \textsuperscript{*}\,Corresponding authors}%
}%
}

\begin{document}

\maketitle

\fancyhead{}

\begin{abstract}
Training large language models for complex reasoning tasks relies on two complementary signals: sparse, outcome-grounded reinforcement learning (RL) for correctness, and dense, token-level distillation for guidance. However, complementarity in supervision does not guarantee compatibility in optimization; imitation-focused distillation can conflict with the outcome-driven RL objective, and existing hybrids lack a principled resolution. We propose \textbf{GIFT} (\textbf{G}radient-\textbf{I}nformed \textbf{F}il\textbf{T}ering for Distillation-Augmented Reasoning RL), a novel method that directly addresses the allocation problem of which teacher updates to apply. GIFT performs a token-level geometric analysis, decomposing the distillation gradient into a component \textbf{parallel} to the RL gradient, which measures first-order conflict and provides a \textbf{safety} filter, and a component \textbf{orthogonal} to it, which quantifies novel directional guidance and provides a measure of \textbf{value}. This decomposition yields a principled scoring rule to rank token-level guidance. We combine this with a progressive \textbf{coverage annealing} schedule that shrinks the budget of tokens receiving distillation, ensuring that the most conflicting updates are discarded first while the most complementary ones are retained at full strength. This facilitates a smooth transition from dense, geometrically-filtered guidance to a pure, outcome-grounded objective at convergence. Extensive experiments on mathematical and logical reasoning benchmarks show that GIFT consistently outperforms strong baselines across multiple model scales, achieving relative gains of up to 9.5\% on difficult mathematical reasoning tasks. Code: \url{https://github.com/lilifree/GIFT}
\end{abstract}



\section{Introduction}
\vspace{-1ex}
\begin{quote}
\itshape
``Thoughts without content are empty, intuitions without concepts are blind. Only through their union can knowledge arise.''
\par
\noindent\hfill\normalfont---\ Immanuel Kant,
\textit{Critique of Pure Reason}
\end{quote}
\vspace{-1.5ex}

Training large language models for complex reasoning relies on two complementary forms of supervision: \textbf{outcome-grounded reinforcement learning and dense, token-level distillation}. Reinforcement learning with verifiable rewards (RLVR), such as GRPO (\cite{grpo}), anchors learning to task success, but offers little guidance about the intermediate decisions that produce an outcome. A single reward supervises an entire response, and GRPO's group-relative signal vanishes when all responses receive the same reward. On-policy self-distillation (OPSD; \cite{opsd}) provides more direct guidance: the same model, conditioned on privileged information such as a reference solution, supplies teacher distributions along the student's own rollout. Yet \textbf{this dense supervision targets imitation rather than correctness}; following it can introduce information leakage and late-stage instability~\citep{rlsd,wei2026amrsdasymmetricmetareflectiveselfdistillation}. We seek \textbf{the best of both worlds: dense teacher guidance to accelerate learning, with verifiable outcomes determining what the model ultimately optimizes}.

Realizing this promise requires more than combining the two signals: \textbf{complementarity in supervision does not guarantee compatibility in optimization}. RLSD and RLCSD use teacher-derived signals to modulate verifier-based advantages (\cite{rlsd}; \cite{pan2026rlcsd}), while OPDVR gates sampled-token distillation rewards by trajectory correctness (\cite{opdvr}). For a nonzero RL gradient at a fixed token, such scalar changes cannot introduce a direction outside its logit-space span. An auxiliary distributional loss, as in SDPG (\cite{sdpg}), can supply such directions, yet admitting a response or token does not establish how its teacher gradient interacts with the current RL update. Teacher withdrawal addresses persistence: RLSD and SDPG anneal teacher influence, and \emph{Sequential Beats Joint} separates distillation and RLVR into successive stages (\cite{li2026sequentialbeatsjointinterplay}). But withdrawal alone leaves an allocation problem. Uniformly weakening or abruptly removing guidance can discard complementary updates together with redundant or conflicting ones. Selecting surviving tokens by disagreement alone can retain a strong teacher signal precisely where it obstructs the outcome update. \textbf{Which token-level teacher contributions should survive as the available supervision budget contracts}?

We address this allocation problem with our method, \textbf{GIFT} (\textbf{G}radient \textbf{I}nformed
\textbf{F}il\textbf{T}ering for Distillation Augmented Reasoning RL), which uses token-level gradient geometry to determine where auxiliary distillation participates. Instead of relying on heuristics like disagreement or uncertainty, GIFT directly analyzes the interaction between the teacher's distillation gradient and the student's RL gradient. At each token, we decompose the teacher's update into two key components: (1) a component \textbf{parallel} to the RL gradient, which measures first-order conflict or assistance, and (2) a component \textbf{orthogonal} to the RL gradient, which represents novel directional guidance that cannot be produced by simply rescaling the RL update.
\textbf{Our central hypothesis is that the most effective guidance is both non-conflicting and contains significant novel direction}. However, since privileged teacher guidance (as in OPSD) targets imitation rather than guaranteed correctness, its influence must be carefully managed to avoid biasing the final policy. GIFT addresses this by implementing its geometric selection under a progressively shrinking token budget. This \textbf{annealing} process ensures that as training matures, the model transitions from dense but potentially flawed guidance towards the verifiable outcome signal. Only the most geometrically complementary teacher contributions survive before the model converges to the pure, unprivileged RL objective.
This approach leads to a method that is both principled and effective. Our main contributions are:
\begin{itemize}[leftmargin=*, itemsep=0pt, topsep=0pt, partopsep=0pt]
\item We propose GIFT, a method that uses the geometric relationship between teacher and student gradients to selectively apply distillation. It is the first method to explicitly distinguish between conflicting, redundant, and complementary token-level guidance to solve the allocation problem.
\item We derive a closed-form, token-level scoring rule based on this geometry, allowing for efficient selection without needing to compute auxiliary backward passes. We combine this with a progressive coverage annealing schedule that smoothly transitions from dense guidance to pure reinforcement learning, ensuring stability and preserving the primacy of the outcome-based objective.
\item GIFT achieves the highest average scores on challenging mathematical reasoning benchmarks across four model scales, outperforming eight strong baselines. Ablation studies confirm that its combination of geometric selection and coverage annealing is critical to its success, yielding significant gains over simpler selection criteria or withdrawal strategies.
\end{itemize}

\section{Related Work}

Dense teacher guidance can supplement outcome supervision, but its usefulness depends on how it is incorporated, not simply on how much is provided. We situate GIFT along four related design choices: how teacher information enters the update, which tokens receive it, how its interaction with RL is assessed, and how its influence is withdrawn.

\textbf{Outcome Supervision and Self-Distillation.} RL from verifiable outcomes, or RLVR, forms the basis of our primary objective. GRPO~\citep{grpo} derives group-relative advantages from verified outcomes; subsequent works like Dr. GRPO~\citep{drgrpo} and DAPO~\citep{dapo} refine its normalization and loss aggregation to improve stability and reduce bias. While these methods provide an unbiased, outcome-grounded learning signal, their supervision is sparse and vanishes for rollout groups where all responses are correct or incorrect. On-policy self-distillation (OPSD; \cite{opsd}) offers a dense, token-level alternative, supplying teacher feedback along the student's trajectory. Hybrid methods, such as RLSD~\citep{rlsd}, RLCSD~\citep{pan2026rlcsd}, and AMR-SD~\citep{wei2026amrsdasymmetricmetareflectiveselfdistillation}, attempt to combine these two signals by using teacher-derived signals to perform \textbf{scalar modulation} of the RL advantages. However, they are fundamentally limited in expressiveness: re-weighting the RL gradient cannot introduce an update direction outside its original logit-space span. GIFT uses OPSD-RLVR geometry to prioritize complementary guidance under exact quotas, without excluding all conflicts.

\textbf{Token-Selective Teacher Guidance.} Given a dense teacher signal, a critical question is which tokens should receive it. Many methods rely on statistical heuristics. TIP~\citep{tip} ranks tokens by student entropy and teacher-student divergence; EOPD~\citep{eopd}) gates supervision by teacher entropy; RSTG~\citep{dstg} combines these heuristics for failed rollouts; and CRPO~\citep{crpo}) uses predictive entropy to distinguish exploration from exploitation bias. While these statistics identify uncertainty or disagreement, they do not directly establish compatibility with the current outcome-driven update. Other works highlight that simple agreement is not the goal. RLRT~\citep{rlrt} reinforces successful student departures from the teacher, demonstrating teacher disagreement can be valuable. SKALD~\citep{SKALD} focuses on zero-variance groups where the RL signal is absent, providing distillation as a fallback. This aligns with GIFT's treatment of such cases (Corollary~\ref{cor:stationary}), but GIFT's geometric criterion is more general. Instead of relying on heuristics, GIFT directly assesses the first-order interaction between the teacher's distillation gradient and student's RL gradient at each token, providing a principled selection basis.


\textbf{Gradient Geometry for Auxiliary Learning.} 
The geometric relationship between gradients from multiple objectives is a well-established concept in multi-task learning. Methods like PCGrad \citep{yu2020gradient} project away conflicting gradient components, while ATTITTUD~\citep{dery2021auxiliary} decomposes an auxiliary gradient into components agreeing, conflicting, or orthogonal to the primary task's gradient subspace, and re-weights them. ATTITTUD is the direct intellectual predecessor to our geometric analysis. However, GIFT makes three critical and novel adaptations for the RLVR context: (1) \textbf{Usage}: We use geometry for discrete selection of tokens, not continuous modification of the auxiliary gradient, which avoids distorting the teacher signal. (2) \textbf{Granularity and Efficiency}: Our analysis is performed at the token level in logit space, where we derive a closed-form, O(M) solution (Proposition 1) for the geometric components, avoiding expensive, parameter-space SVD approximations required by ATTITTUD. (3) \textbf{Value of Orthogonality}: We identify the orthogonal component ($b_i$) as the most valuable form of complementary guidance—representing novel directions unreachable by the RL signal (Theorem~\ref{thm:diversity})—whereas ATTITTUD treats it as neutral.

\textbf{Demonstration-Guided RL and Teacher Withdrawal.} Privileged teacher guidance must be carefully managed to avoid biasing the final policy towards imitation. This withdrawal can be implemented in several ways. Methods like LUFFY~\citep{luffy} and UFT~\citep{UFT} introduce teacher guidance through trajectory composition or annotated prefixes, and anneal the length or proportion of this guidance. Other methods anneal the auxiliary loss coefficient ($\lambda$) to zero, as seen in Kickstarting~\citep{Schmitt2018KickstartingDR}, RLSD, and SDPG. However, such \textbf{coefficient annealing} dilutes the signal for all tokens simultaneously, discarding complementary updates along with conflicting ones. Sequential Beats Joint~\citep{li2026sequentialbeatsjointinterplay} avoids this by separating distillation and RLVR into successive stages, but forgoes the benefits of joint training. GIFT addresses this optimization tension through \textbf{coverage annealing}, shrinking the selected-token quota while fixing the loss weight. At each ranking, positive residual scores precede zero scores and conflicts; within conflicts, smaller normalized opposition receives higher priority. This controls teacher participation while preserving retained-token loss weights, before the auxiliary loss vanishes at teacher exit.

\section{Method}
\label{sec:method}

In the introduction, we established the central challenge: to effectively merge the dense, token-level guidance from self-distillation~\citep{sdpg, gkd, minillm} with the sparse, outcome-grounded supervision from reinforcement learning~\citep{grpo, Guo_2025} (RL). Simply combining them is insufficient; we need a principled mechanism to decide \textbf{when} and \textbf{where} to apply the teacher's guidance to ensure it complements, rather than corrupts, the primary RL objective. This section introduces our method, GIFT, which resolves this "allocation problem" through a geometric analysis of the underlying gradients, as shown in Figure~\ref{fig:overview}.

\subsection{A Geometric Approach to Guidance}

\begin{figure}[t]
\vspace{-0.7ex}
    \centering
    \includegraphics[width=0.95\linewidth]{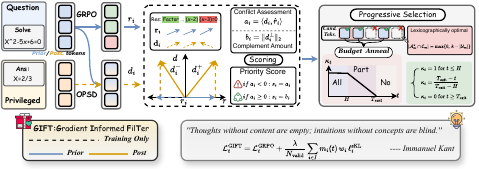}
    \vspace{-1.2ex}
    \caption{GIFT decomposes the distillation gradient into conflict ($a_i$) and complement ($b_i$) components relative to the RL gradient, ranks tokens, and anneals the selection budget from full to zero.}
\vspace{-4.4ex}
\label{fig:overview} 
\end{figure}

Our methodology is built on a direct analysis of the geometric interaction between the teacher and student gradients~\citep{dery2021auxiliary,yu2020gradient} at each token. Let $z_{i} \in \mathbb{R}^{|V|}$ be the student's logits at position $i$, with probability distribution $p_{i} = \text{softmax}(z_{i})$. At this position, we consider two distinct logit-space gradients. The first is the \textbf{outcome-driven gradient} from the clipped GRPO loss~\citep{grpo}, $l_{i}^{\text{GRPO}}$. This gradient, denoted $r_i$, represents the direction of steepest ascent for the RL objective, steering the model toward task success determined by the verifier. The second is the \textbf{imitation-driven gradient} from the distillation loss, $l_{i}^{\text{sKL}}$~\citep{gkd,minillm,opsd,sdpg}. This gradient, denoted $d_i$, points in the direction that best aligns the student's distribution $p_i$ with the privileged teacher's distribution $\tilde{q}_i$. These gradients are:
\begin{equation}
  r_i=\nabla_{z_i}\ell^{\mathrm{GRPO}}_i=w_iA_i\rho_ih_i\bigl(p_i-e_{u_i}\bigr),
  \qquad
  d_i=\nabla_{z_i}\bigl(w_i\ell^{\mathrm{sKL}}_i\bigr)=w_i\bigl(p_i-\tilde q_i\bigr),
  \label{eq:grads}
\end{equation}
where $e_{u_{i}}$ is the one-hot vector for the sampled token, $A_i$ is the advantage, $h_i$ indicates an active clipped branch, and $\tilde{q}_i$ is the teacher distribution. The gradient $r_i$ serves as our reference direction.
On the set of tokens $\mathcal{D}$ where a non-zero RL gradient exists, we decompose the teacher gradient $d_i$ with respect to the normalized RL gradient $\hat{r}_i = r_i / ||r_i||_2$:
\begin{equation}
  a_i=\langle d_i,\hat r_i\rangle,\quad
  d_i^{\perp}=d_i-a_i\hat r_i,\quad
  b_i=\|d_i^{\perp}\|_2.
  \label{eq:decomp}
\end{equation}
Here, the scalar $a_i$ is the signed magnitude of the \textbf{parallel component}, measuring first-order assistance or conflict. The scalar $b_i$ is the norm of the \textbf{orthogonal component}, quantifying novel directional guidance. Together, they satisfy $a_{i}^{2} + b_{i}^{2} = ||d_{i}||_{2}^{2}$. These two scalars form the core of our selection mechanism, as explained next.

\subsection{Safety: The Role of the Parallel Component ($a_i$)}

Our first priority is to ensure that auxiliary guidance does not undermine the primary, outcome-driven objective. The parallel component $a_i$ provides a direct, first-order measure of this safety. Theorem 1 formalizes this intuition.

\begin{theorem}[Local first-order non-interference]
\label{thm:noninterference}
For a given token, the change in the primary GRPO loss from a combined update~\citep{dery2021auxiliary,yu2020gradient} is given by $\frac{d}{d\eta}f_{i}(z_{i}(\eta))|_{\eta=0} = -||r_{i}||_{2}^{2} - \lambda||r_{i}||_{2}a_{i}$. The implications are direct:
\begin{enumerate}[leftmargin=*, itemsep=0pt, topsep=0pt, partopsep=0pt]
    \item \textbf{Assistance ($a_i > 0$)}: The teacher's guidance is aligned with the RL gradient and accelerates the decrease of the primary loss. This guidance is beneficial.
    \item \textbf{Neutrality ($a_i = 0$)}: The teacher's guidance is orthogonal to the RL gradient and has no first-order effect on the primary loss. It is safe, regardless of its magnitude $||d_i||_2$.
    \item \textbf{Conflict ($a_i < 0$)}: The teacher's guidance opposes the RL gradient, slowing or even reversing progress on the primary objective. This guidance is harmful.
\end{enumerate}
\end{theorem}
This geometric criterion provides a more precise filter than heuristics like student-teacher divergence (e.g., TIP~\citep{tip}) or disagreement. While a large divergence might indicate valuable guidance, it does not distinguish between productive, orthogonal disagreement ($a_i \ge 0, b_i \gg 0$) and destructive, head-on conflict ($a_i < 0$). Our use of $a_i$ makes this crucial distinction.
Therefore, the sign of $a_i$ serves as a fundamental \textbf{safety gate}. We define any token with $a_i < 0$ as conflicting and give it lowest priority. This raises the next question: among the "safe" tokens (where $a_i \ge 0$), which are most \textbf{valuable} to select? This is decided by the orthogonal component, $b_i$.


\subsection{Value: Why the Orthogonal Component ($b_i$) Matters}


Having established safety with the parallel component $a_i$, we now examine value. Among ``safe'' tokens ($a_i \ge 0$), the most useful guidance lies in the orthogonal component, quantified by $b_i$. Crucially, $b_i$'s value does not stem from a direct, first-order contribution to the primary RL objective—as established in Theorem~\ref{thm:noninterference}, its contribution to the primary loss descent is zero. Its profound importance relates instead to the guidance's novelty and expressiveness. 
This value is best understood through its counterpart. A purely parallel teacher gradient ($b_i=0$) is fundamentally redundant. The resulting update, $r_i + \lambda d_i$, simplifies to merely rescaling the original RL gradient $r_i$. This critical limitation offers no new directional information unattainable by simply adjusting the learning rate or advantage scaling for that token. An entirely collinear teacher tells the optimizer nothing it could not learn from the RL signal alone.

All genuinely new geometric information provided by the teacher resides in the orthogonal component $d_i^\perp$. Therefore, selecting for a large $b_i$ is equivalent to prioritizing the most valuable form of teacher guidance—the part that is structurally inaccessible to simpler advantage-modulation schemes. The following theorem formalizes this concept of complementarity in three distinct ways.

\begin{theorem}[Complementarity of the orthogonal residual]
\label{thm:diversity}
Under the assumptions of Theorem~\ref{thm:noninterference}:
\begin{enumerate}[leftmargin=*, itemsep=0pt, topsep=0pt, partopsep=0pt]
    \item \textbf{(Maximal locally GRPO-neutral distillation decrease.)}
    Among unit-norm logit perturbations that are locally neutral to the token-level
    primary surrogate,
    \begin{equation}
        \max_{\lVert v\rVert_2\le 1,\; r_i^\top v=0}
        \bigl(-d_i^\top v\bigr)=b_i.
    \end{equation}
    If $b_i>0$, the maximum is attained by
    $v^*=-d_i^\perp/b_i$. If $b_i=0$, every feasible $v$ attains the value zero.
    Thus, $b_i$ is the maximal first-order distillation decrease among unit-norm
    logit perturbations that are locally neutral to the token-level primary
    surrogate; no finite-step or shared-parameter conclusion follows.
\item \textbf{(Irreducibility to token-local GRPO rescaling.)} Fix the sampled tokens, model state, and reference gradients, and equip $\mathcal H=\bigoplus_{i\in\mathcal I}\mathbb R^{|\mathcal V|}$ with the
Euclidean product norm $\|(x_i)_i\|_{\mathcal H}^2=\sum_i\|x_i\|_2^2$.
Let $\mathcal M_0 =\left\{(c_i r_i)_{i\in\mathcal I}:c_i\in\mathbb R\right\}$ be the space of unrestricted signed, token-local rescalings of the fixed GRPO reference field. For any fixed binary mask $m\in\{0,1\}^{\mathcal I}$, write $\mathcal D_m=(m_i d_i)_{i\in\mathcal I}$. Then
\begin{equation}
\begin{split}
    \operatorname{dist}^2(\mathcal D_m,\mathcal M_0)
    ={}&\sum_{i:r_i\ne 0}m_i
    \left(\|d_i\|_2^2-
    \frac{\langle d_i,r_i\rangle^2}{\|r_i\|_2^2}\right)
    +\sum_{i:r_i=0}m_i\|d_i\|_2^2.
\end{split}
\end{equation}
In particular, for $\epsilon_g\ge0$ and
$\mathcal D_\epsilon=\{i:\|r_i\|_2>\epsilon_g\}$, the residuals defined in Equation~\ref{eq:decomp} satisfy
\begin{equation}
    \operatorname{dist}^2(\mathcal D_m,\mathcal M_0)
    \ge \sum_{i\in\mathcal D_\epsilon}m_i b_i^2.
\end{equation}
Thus, selecting any $i\in\mathcal D_\epsilon$ with $b_i>0$ certifies that the selected teacher field cannot be represented by token-local scalar rescalings of the fixed GRPO field. This is a sufficient, not a necessary, condition. A selected zero-reference block with $d_i\ne0$ also certifies irreducibility, as can a noncollinear small nonzero-reference block.

  \item \textbf{(Gradient diversity.)}
    For a collection of gradients $\{g_j\}_{j=1}^n$ summed into one update,
    define
    \begin{equation}
      \Delta\bigl(\{g_j\}\bigr)
      =\sum_{j}\|g_j\|_2^{2}/\bigl\|\sum_{j}g_j\bigr\|_2^{2}
      \label{eq:diversity}
    \end{equation}
    (Definition~\ref{def:diversity}; $\Delta=1/n$ when the $g_j$ are identical
    and $\Delta=1$ when they are mutually orthogonal, so larger $\Delta$ means
    less redundancy among the summands). Applying this to the two gradients
    combined at token $i$, i.e.\ $\Delta_i=\Delta(\{r_i,\lambda d_i\})$: at fixed
    $\|r_i\|_2$, $\|d_i\|_2$ and $\lambda>0$, $\Delta_i$ is strictly decreasing
    in $a_i$, hence \emph{strictly increasing in $b_i$ on the safe branch}
    $a_i\in[0,\|d_i\|_2]$, where it rises from $(\|r_i\|_2^2+\lambda^2\|d_i\|_2^2)/(\|r_i\|_2+\lambda\|d_i\|_2)^2\in[\tfrac12,1)$
    at $b_i=0$ to $1$ at $a_i=0$. Unconstrained over all $a_i$, however,
    $\Delta_i$ is maximized as $a_i\to-\|d_i\|_2$ and diverges when
    $r_i+\lambda d_i=0$.
\end{enumerate}
\end{theorem}
This theorem justifies prioritizing the orthogonal component $b_i$ for three reasons. First, $b_i$ measures maximal progress on the imitation objective that is ``free'' to the outcome-driven task, having no first-order cost to the primary RL objective. Second, it represents guidance advantage-modulation methods like RLSD~\citep{rlsd} and OPDVR~\citep{opdvr} cannot express, since they only rescale the RL gradient without introducing novel directions. Third, in sparse-reward RLVR~\citep{grpo,Guo_2025}, a large $b_i$ indicates high gradient diversity ($\Delta_i$), injecting new dimensions to prevent stagnation. However, the most diverse signal directly opposes the primary gradient ($a_i < 0$), which would be catastrophic. Our method resolves this by assigning conflicting signals negative scores, ensuring their early discard as the token budget anneals. This provides early dense guidance while progressively enforcing safety.

GIFT resolves this tension by framing token selection as an  
\emph{exact-quota lexicographic} program: at each step, a quota of $k$ eligible tokens must be selected, and GIFT chooses the size-$k$ subset of the eligible set $E_m$ that lexicographically minimizes
\begin{equation}
\label{eq:lexprogram}  
A^{\star}_m \, \in \,  \arg\min\nolimits^{\mathrm{lex}}_{A \subseteq E_m,  
|A| = k} \Big( N^{-}(A),\, O^{-}(A),\, -E^{+}(A) \Big),
\end{equation}
where $C_m = \{i \in E_m \cap \mathcal D : a_i < 0\}$ collects the conflicting eligible tokens, $U_m = E_m \setminus C_m$ the remaining eligible ones, $N^{-}(A) = |A \cap C_m|$ their number, $O^{-}(A) = \sum_{i \in A \cap  
C_m} |a_i|$ their \emph{normalized opposition}, and $E^{+}(A) = \sum_{i \in A \cap U_m \cap D} b_i^2$ the injected orthogonal guidance.


\subsection{Scoring and Progressive Selection}


The lexicographic exact-quota program in Equation~\ref{eq:lexprogram}
admits an exact, closed-form solution in the form of a piecewise scoring
rule, providing a principled ranking for all eligible tokens.
Each token receives a score $s_i$ defined as:
\begin{equation}
  s_i=
  \begin{cases}
    a_i, & i\in\mathcal{D},\ a_i<0,\\
    b_i, & i\in\mathcal{D},\ a_i\ge0,\\
    0,   & i\in\mathcal{E}\setminus\mathcal{D},
  \end{cases}
  \label{eq:score}
\end{equation}
This structure imposes a strict hierarchy that directly reflects our optimization criteria. For tokens with conflicting guidance ($a_i < 0$), the score is negative, ensuring they are ranked lowest. For all safe tokens ($a_i \ge 0$), the score is simply the orthogonal magnitude $b_i$, prioritizing novel, complementary information. Tokens with no usable RL gradient ($i \notin \mathcal{D}$) receive a neutral score of zero.
The optimality of this ranking is formally guaranteed by the following theorem.
\begin{theorem}[Exact-quota lexicographic optimality]
\label{thm:selection}
With the notation of Equation~\ref{eq:lexprogram}, for any integer quota
$k \in \{0, \ldots, |E_m|\}$, the set $A^{\star}_m$ of $k$ highest scores
lexicographically minimizes $(N^{-}, O^{-}, -E^{+})$ over size-$k$
subsets, and
\begin{equation}
|A^{\star}_m \cap C_m| = \max\{0,\, k - |U_m|\}.
\label{eq:gift-selected-conflicts}
\end{equation}
If $k \le |U_m|$ the selection is conflict-free and maximizes $E^{+}$
over size-$k$ subsets of $U_m$; otherwise it takes $U_m$ plus the
$k - |U_m|$ conflicting tokens of smallest normalized opposition
$|a_i|$. Both hold under arbitrary tie-breaking.
\end{theorem}
Theorem~\ref{thm:selection} provides the formal guarantee that this scoring rule is the exact solution to the stated lexicographic exact-quota problem: it first minimizes the number of conflicting tokens admitted, then their normalized opposition, and only then maximizes the injected orthogonal  
guidance ($\sum_i b_i^2$).


This static ranking is embedded within a dynamic withdrawal strategy, implemented by progressively shrinking the token budget $k_t$ over the course of training via an annealing schedule.
\begin{equation}
\kappa_t = \frac{T_{\text{exit}} - t}{T_{\text{exit}} - H} \quad \text{for } H < t < T_{\text{exit}},
\label{eq:kappa}
\end{equation}
where $\kappa_t$ is the retention ratio, $H$ is an initial full-coverage phase, and $T_{\text{exit}}$ is the step at which all teacher guidance is withdrawn. This annealing of the \textbf{coverage}, rather than a loss coefficient, makes the trade-off between exploration and exploitation time-varying. Early in training, when the budget is large and the RL signal is sparse, the model benefits from broad, diverse guidance. As training matures and the budget contracts, the lowest-scoring tokens—the most conflicting ones—are the first to be excluded. The most complementary tokens (high $b_i$, $a_i \ge 0$) are retained the longest.

This provides a principled, state-aware withdrawal order. Unlike uniform coefficient annealing~\citep{sdpg, rlsd} which dilutes all guidance simultaneously, or staged training~\citep{li2026sequentialbeatsjointinterplay} which forgoes joint optimization, GIFT's geometric ranking and coverage annealing ensures that as the supervision budget contracts, it is precisely the most harmful guidance (lowest-scoring tokens) that is shed first, while the most complementary guidance is retained at full strength.

\subsection{The GIFT Objective}

The principles of geometric selection and progressive withdrawal are implemented in the final GIFT objective function:
\begin{equation}
  \mathcal{L}^{\mathrm{GIFT}}_t=\mathcal{L}^{\mathrm{GRPO}}_t
  +\frac{\lambda}{N_{\mathrm{valid}}}\sum_{i\in I}m_i(t)\,w_i\,\ell^{\mathrm{sKL}}_i,
  \qquad \lambda=1 .
  \label{eq:objective}
\end{equation}
Here, $m_i(t)$ selects the top-$k_{t,m}$ eligible tokens per sub-batch; $\mathcal I$ contains all valid positions and $N_{\rm valid}=|\mathcal I|$.
Several critical design choices apply.
First, fixing $\lambda=1$ ensures selected tokens receive full-strength, unprojected distillation loss. Control operates via \textbf{selection} (who receives guidance), not \textbf{dilution} (guidance strength). The teacher loses reach, not voice.
Second, normalizing over all valid tokens ($N_{\text{valid}}$) rather than just selected ones ensures stability, keeping each retained token's contribution independent of budget size. Otherwise, a shrinking budget would artificially inflate surviving tokens, undermining smooth withdrawal.
Finally, as the budget anneals to zero ($t \ge T_{\text{exit}}$), the mask $m_i(t)$ becomes universally zero, guaranteeing a clean transition. The auxiliary term vanishes, reducing $\mathcal{L}_t^{\text{GIFT}}$ identically to the native GRPO loss. This grounds final behavior purely in the outcome-driven objective. The teacher acts as a temporary, geometrically-filtered scaffold, harnessing dense guidance before complete withdrawal without corrupting the final policy.

\section{Experiments}


\providecommand{\giftgain}[1]{\,{\scriptsize\textcolor{green!45!black}{(+#1)}}}
\begin{table*}[t]
\vspace{-0.5ex}
\centering
\caption{Mathematical reasoning across Qwen3 scales (mean@12, \%). Avg. is the mean over AMC23, AIME24, and AIME25. Green denotes relative gains over BASE; bold and underlined values indicate the best and second-best distinct scores. Bench. denotes benchmark.
}
\label{tab:gift-math}
\vspace{-1.5ex}
\resizebox{0.98\textwidth}{!}{%
\begin{tabular}{llccccccccc>{\columncolor{gray!12}}c}
\toprule
Size & Bench. & BASE & GRPO & OPSD & SDPO & RLSD & RLCSD & OPDVR & LUFFY & UFT & \textbf{GIFT} \\
\midrule
\multirow{4}{*}{1.7B} & AMC23 & 75.50 & 75.70 & 75.60 & 75.30 & 75.30 & \underline{75.80} & \textbf{76.51} & 75.40 & 74.60 & \underline{75.80}\giftgain{0.40\%} \\
& AIME24 & 44.44 & 45.28 & 45.56 & 44.72 & 44.72 & 46.94 & 46.67 & \underline{48.33} & \underline{48.33} & \textbf{48.61}\giftgain{9.38\%} \\
& AIME25 & 35.00 & 35.56 & 33.33 & \underline{36.67} & 34.72 & 34.72 & 35.83 & 34.17 & \underline{36.67} & \textbf{38.33}\giftgain{9.51\%} \\
& Avg. & 51.65 & 52.18 & 51.50 & 52.23 & 51.58 & 52.49 & 53.00 & 52.63 & \underline{53.20} & \textbf{54.25}\giftgain{5.03\%} \\
\midrule
\multirow{4}{*}{4B} & AMC23 & 88.45 & 89.16 & 88.25 & 88.96 & 88.86 & \textbf{90.26} & 88.45 & 89.16 & \underline{89.66} & 88.55\giftgain{0.11\%} \\
& AIME24 & 71.67 & \underline{74.17} & 73.06 & 70.83 & 73.61 & 72.22 & 73.61 & \textbf{76.11} & 70.56 & 73.33\giftgain{2.32\%} \\
& AIME25 & 63.89 & 63.61 & \underline{67.22} & \underline{67.22} & 63.61 & 64.17 & 64.72 & 65.28 & 64.72 & \textbf{69.72}\giftgain{9.13\%} \\
& Avg. & 74.67 & 75.65 & 76.18 & 75.67 & 75.36 & 75.55 & 75.59 & \underline{76.85} & 74.98 & \textbf{77.20}\giftgain{3.39\%} \\
\midrule
\multirow{4}{*}{8B} & AMC23 & 89.16 & 89.36 & 88.76 & 89.56 & 88.76 & 89.26 & \textbf{89.96} & \textbf{89.96} & 88.86 & \underline{89.66}\giftgain{0.56\%} \\
& AIME24 & 75.00 & 74.72 & \textbf{77.22} & 75.83 & 75.28 & \underline{76.11} & 74.44 & 73.33 & \underline{76.11} & \underline{76.11}\giftgain{1.48\%} \\
& AIME25 & 65.83 & 65.83 & 66.39 & 66.39 & \underline{69.72} & 68.06 & \underline{69.72} & 67.50 & 66.39 & \textbf{71.11}\giftgain{8.02\%} \\
& Avg. & 76.67 & 76.64 & 77.46 & 77.26 & 77.92 & 77.81 & \underline{78.04} & 76.93 & 77.12 & \textbf{78.96}\giftgain{2.99\%} \\
\bottomrule
\end{tabular}}
\vspace{-4.5ex}
\end{table*}

We conduct experiments to validate GIFT's core principles. Our evaluation addresses three central questions: (1) \textbf{Effectiveness}: Does GIFT's approach of gradient-informed filtering and progressive withdrawal achieve state-of-the-art performance on challenging reasoning tasks compared to existing methods? (2) \textbf{Generalization}: Do GIFT's benefits extend across different architectures, scales, and reasoning domains, including out-of-distribution scenarios? (3) \textbf{Component Necessity}: Are GIFT's key components—specifically the geometric scoring rule based on conflict ($a_i$) and novelty ($b_i$), and the coverage annealing schedule—critical to its success? 

\subsection{Experimental Setup}
\label{sec:gift-experimental-setup}

\textbf{Models and tasks.}
We evaluate on two families of models: three scales of Qwen3 (1.7B, 4B, 8B)~\citep{yang2025qwen3technicalreport}) and OLMo-3-7B-Think~\citep{olmo2026olmo3}. For mathematical reasoning, we train on the DeepMath-103K dataset~\citep{he2025deepmath103klargescalechallengingdecontaminated}, using progressively harder difficulty splits for larger models, and evaluate on the AMC23~\citep{mathai_amc23}, AIME24~\citep{aime24}, and AIME25~\citep{aime25} benchmarks. For logical reasoning, we train on Knights-and-Knaves puzzles~\citep{xie2025logicrlunleashingllmreasoning} (4-8 roles) and evaluate on both in-distribution (ID) and out-of-distribution (OOD, 9-11 roles) splits to test generalization.

\textbf{Baselines.}
We compare GIFT with the initial model (BASE) and eight training baselines, grouped by supervision source and how it enters optimization (Appendix~\ref{app:baselines}): GRPO, using verifier-derived advantages without teacher feedback~\citep{grpo}; distributional self-distillation, matching next-token distributions of a conditioned self-teacher (OPSD~\citep{opsd}, SDPO~\citep{sdpo}); verifier-guided token-level credit assignment, modulating advantages or gating rewards with teacher-derived scalars (RLSD~\citep{rlsd}, RLCSD~\citep{pan2026rlcsd}, OPDVR~\citep{opdvr}); and demonstration-guided methods (LUFFY~\citep{luffy}, UFT\citep{UFT}).

\textbf{Evaluation.}
For mathematics, we report mean@12, the average correctness over 12 sampled responses per problem, averaged across problems, using rule-based answer matching. For logic, we report pass@1 on 500 ID problems and 100 problems for each
of the three OOD role counts (9, 10, and 11).
The reported Avg. equally weights the three mathematical benchmarks or the four logical-reasoning splits. All scores are percentages, and gains are relative percentage-point differences.



\subsection{Main Results}
\label{sec:gift-main-results}

Our main results demonstrate that GIFT's principled filtering mechanism consistently outperforms existing methods across multiple tasks and model scales, providing a superior solution to the core tension between dense guidance and outcome-grounded learning.

\begin{table*}[t]
\vspace{-0.7ex}
\centering
\caption{Logical reasoning on Knights-and-Knaves (pass@1, \%).
The ID split contains 4--8 roles, while the OOD splits contain 9, 10, and 11 roles.
Green denotes relative gains over BASE.}
\label{tab:gift-logic}
\vspace{-1.5ex}
\resizebox{\textwidth}{!}{%
\begin{tabular}{lccccccccc>{\columncolor{gray!12}}c}
\toprule
Benchmark & BASE & GRPO & OPSD & SDPO & RLSD & RLCSD & OPDVR & LUFFY & UFT & \textbf{GIFT} \\
\midrule
\multicolumn{11}{l}{\textit{In-distribution (ID)}} \\
\quad 4--8 roles & 72.40 & 73.40 & 72.80 & \textbf{75.20} & 73.60 & 74.00 & 73.00 & \underline{74.60} & 74.20 & 73.40\giftgain{1.38\%} \\
\midrule
\multicolumn{11}{l}{\textit{Out-of-distribution (OOD)}} \\
\quad 9 roles & 69.00 & 72.00 & \textbf{75.00} & \underline{73.00} & 71.00 & 69.00 & 71.00 & 67.00 & 71.00 & \underline{73.00}\giftgain{5.80\%} \\
\quad 10 roles & \underline{59.00} & 56.00 & \underline{59.00} & 56.00 & 55.00 & \underline{59.00} & 57.00 & \underline{59.00} & 56.00 & \textbf{61.00}\giftgain{3.39\%} \\
\quad 11 roles & 42.00 & \textbf{46.00} & 41.00 & 40.00 & 41.00 & \underline{44.00} & \underline{44.00} & 43.00 & 40.00 & \underline{44.00}\giftgain{4.76\%} \\
\midrule
Avg. & 60.60 & 61.85 & \underline{61.95} & 61.05 & 60.15 & 61.50 & 61.25 & 60.90 & 60.30 & \textbf{62.85}\giftgain{3.71\%} \\
\bottomrule
\end{tabular}}
\vspace{-3.5ex}
\end{table*}

\textbf{Mathematical Reasoning.}
As shown in Table~\ref{tab:gift-math}, GIFT achieves the highest average score on mathematical reasoning across all three Qwen3 models. Compared to the base model, it achieves average relative gains of 5.03\%, 3.39\%, and 2.99\% for the 1.7B, 4B, and 8B models, respectively. The inferiority of the baselines stems from their distinct, yet suboptimal, strategies for combining supervision.

The two extremes of the supervision spectrum, pure RL and pure distillation, highlight the core challenge. GRPO provides an unbiased, outcome-grounded learning signal, which is why we retain it as our primary objective. However, its supervision is sparse—a single reward for an entire trajectory—and its group-relative signal vanishes in uniformly correct or incorrect groups, leading to sample inefficiency and learning stagnation. Conversely, on-policy self-distillation offers dense, token-level guidance but optimizes for imitation of a privileged teacher rather than verifiable correctness: OPSD conditions its self-teacher on a reference solution, while SDPO conditions it on environment feedback or verified successful rollouts. In both, the auxiliary loss is applied unconditionally along the entire trajectory—with no check for compatibility with the outcome-driven update and no mechanism for withdrawing teacher influence.

Hybrid methods attempt to bridge this gap, but their approaches lead to distinct compromises. One family, including RLSD, RLCSD, and OPDVR, ultimately converts the teacher signal into a scalar that modulates the RL update: RLSD and RLCSD use teacher-derived signals to modulate verifier-based advantages, while OPDVR reformulates the implicit reward of sampled-token distillation based on trajectory correctness and applies ReLU gating, so that the teacher enters the update only through a per-token scalar reward. By collapsing the teacher's rich distributional information into a scalar reward or weight, these methods are structurally incapable of introducing novel update directions, as established in our theoretical analysis (Theorem~\ref{thm:diversity}(ii)); they can only re-weight the token-level components of the existing RL update, whereas GIFT's auxiliary loss can provide valuable orthogonal guidance ($b_i$). Other methods like LUFFY and UFT introduce teacher guidance through trajectory composition or annotated prefixes, which likewise lack the fine-grained, geometry-aware conflict checking that GIFT performs at every token.

In contrast, GIFT synthesizes the strengths of these approaches: it complements GRPO's reliable but sparse signal with dense, expressive guidance from an auxiliary loss, but subjects that guidance to a rigorous geometric filter. This unique combination explains its consistent, high-level performance, particularly on the challenging AIME25 benchmark where it establishes a new state-of-the-art.

\textbf{Logical Reasoning and Generalization.} On the Knights-and-Knaves task (Table~\ref{tab:gift-logic}), GIFT again achieves the highest average score, improving over the base model with a relative gain of 3.71\%. Crucially, this improvement is not confined to the in-distribution test set; GIFT shows a substantial average relative gain of 4.70\% on the OOD splits. This indicates that the reasoning capabilities learned via gradient-informed guidance generalize to more complex problems not seen during training. The consistent failure of baselines to match this performance further underscores our analysis: methods that are either structurally limited (scalar modulation), ungrounded (pure imitation), or lack a robust filtering mechanism (unfiltered or coarse-grained guidance) cannot as effectively harness teacher guidance to build generalizable reasoning skills. This result, combined with strong performance on the OLMo-3 model family (Appendix~\ref{appendix: additional_experi}, Figure~\ref{fig:gift-olmo}), confirms that GIFT's mechanism is not specific to a single model architecture or reasoning domain but offers a generally applicable framework for improving complex reasoning.

\begin{table*}[t]
\vspace{-0.5ex}
\centering
\caption{
Ablation of token-selection strategies.
Mathematical reasoning uses Qwen3-8B and is evaluated by mean@12.
Logical reasoning uses Qwen3-4B and is evaluated by pass@1.
}
\label{tab:selection_ablation}
\vspace{-1.5ex}
\resizebox{0.99\textwidth}{!}{%
\begin{tabular}{lccccccccc}
\toprule
\multirow{2}{*}{Method}
& \multicolumn{4}{c}{\textbf{Mathematical Reasoning} (mean@12)}
& \multicolumn{5}{c}{\textbf{Logical Reasoning} (pass@1)}
\\
\cmidrule(lr){2-5}
\cmidrule(lr){6-10}
& AMC23 & AIME24 & AIME25 & Avg.
& ID (4--8) & OOD-9 & OOD-10 & OOD-11 & Avg.
\\
\midrule

All Tokens
& \textbf{89.96} & 76.11 & 66.39 & 77.49
& \textbf{73.80} & 69.00 & 59.00 & 45.00 & 61.70
\\

Random Tokens
& 89.66 & 75.28 & 66.67 & 77.20
& 73.60 & 71.00 & 58.00 & 46.00 & 62.15
\\

Entropy-KL Selection
& 89.36 & 75.28 & 66.11 & 76.92
& 71.80 & 63.00 & 54.00 & \textbf{52.00} & 60.20
\\

Gradient-cosine Selection
& 89.56 & \textbf{77.50} & 65.83 & 77.63
& 72.80 & 68.00 & 61.00 & 39.00 & 60.20
\\

\midrule

Full Model
& 89.66 & 76.11 & \textbf{71.11} & \textbf{78.96}
& 73.40 & \textbf{73.00} & \textbf{61.00} & 44.00 & \textbf{62.85}
\\

\bottomrule
\end{tabular}%
}
\vspace{-1ex}
\end{table*}

\begin{table*}[t]
\centering
\begin{minipage}[t]{0.6\textwidth}
\centering
\caption{Ablation of selection strategies on Qwen3-8B. Mean@12 scores on AMC23, AIME24, and AIME25.}
\label{tab:selection_strategy}
\vspace{-1.5ex}
\resizebox{\textwidth}{!}{%
\begin{tabular}{l cccc c}
\toprule
Benchmark & Conflict-only & Residual-only & Teacher-norm & Conflict-Free & \textbf{GIFT} \\
\midrule
AMC23   & 89.86 & 89.26 & 88.76 & \textbf{89.96} & 89.66 \\
AIME24  & 75.56 & 75.00 & 73.89 & \textbf{76.11} & \textbf{76.11} \\
AIME25  & 65.83 & 66.67 & 68.06 & 67.11 & \textbf{71.11} \\
\midrule
Avg.    & 77.08 & 76.97 & 76.90 & 77.73 & \textbf{78.96} \\
\bottomrule
\end{tabular}%
}
\end{minipage}%
\hfill 
\begin{minipage}[t]{0.35\textwidth}
\centering
\caption{Ablation of token selection {scopes} on Qwen3-8B.}
\label{tab:selection_scope}
\vspace{-1.5ex}
\resizebox{\textwidth}{!}{%
\begin{tabular}{l cc c}
\toprule
Benchmark & Per-rollout & Per-prompt & \textbf{GIFT} \\
\midrule
AMC23   & 89.09 & 89.56 & \textbf{89.66} \\
AIME24  & 75.28 & \textbf{76.39} & 76.11 \\
AIME25  & 67.78 & 66.39 & \textbf{71.11} \\
\midrule
Avg.    & 77.37 & 77.45 & \textbf{78.96} \\
\bottomrule
\end{tabular}%
}
\end{minipage}
\vspace{-3ex}
\end{table*}

\subsection{Ablation Study}
\label{sec:gift-token-selection-ablation}
Here We present key findings from the ablations detailed in Appendix~\ref{app:gift-ablations}.


\textbf{The Geometric Scoring Rule is Essential.}  Table~\ref{tab:selection_ablation} shows heuristic criteria (e.g., entropy-KL divergence, raw gradient cosine similarity) fail to match GIFT, confirming statistical proxies for disagreement are less effective than our direct geometric analysis. Table~\ref{tab:selection_strategy} shows its constituent parts are necessary, though their importance varies with task difficulty. On simpler AMC23 and AIME24 benchmarks, ``Conflict-Free'' and ``Conflict-only'' ablations perform comparably or better than GIFT. When the primary RL gradient is reliable, simpler strategies prioritizing safety by strictly rejecting conflicts (``Conflict-Free'') or accelerating agreement (``Conflict-only'') are highly effective. However, these strategies falter on the challenging AIME25 benchmark, where GIFT excels. The ``Residual-only'' ablation—prioritizing novel guidance ($b_i$) while ignoring conflict signs—underperforms across all settings, demonstrating pursuing novelty without safety checks is detrimental. Conversely, ``Conflict-Free'' proves too restrictive on AIME25, likely discarding valuable teacher guidance needed to escape local optima simply because it appears conflicting. GIFT’s superiority on AIME25 stems from its synthesis: using the conflict signal ($a_i$) for safety while leveraging its exact-quota system to judiciously admit the least-opposing conflicting tokens when necessary, providing a vital mechanism for exploration. These results confirm GIFT's complete geometric scoring rule provides a more robust approach than isolated heuristics, validating a design that places conflicting tokens below zero- and positive-score candidates.

\textbf{Coverage Annealing is Superior to Alternative Anealing Strategies.} Withdrawing teacher influence is as critical as applying it. We thus compare GIFT's \textbf{coverage annealing} with two alternatives. First, \textbf{sequential training}~\citep{li2026sequentialbeatsjointinterplay} trains with OPSD distillation before switching to pure GRPO. Second, \textbf{weight annealing} combines objectives but uniformly applies distillation, gradually decaying the loss coefficient $\lambda$ to zero. As shown in Table~\ref{tab:gift-schedule} (Appendix~\ref{appendix: additional_experi}), GIFT significantly outperforms both strategies, averaging 1.59 and 2.87 percentage point gains, respectively. Unlike sequential training, GIFT benefits from joint optimization synergy. Furthermore, unlike indiscriminate weight annealing, GIFT's coverage annealing preserves the most valuable teacher signals at full strength while selectively pruning unhelpful ones. This confirms that selective, state-aware withdrawal is superior to both staged training and uniform decay.

\textbf{A Broader Allocation Scope is More Effective.}
We find that allocating the token budget across the entire \textbf{training sub-batch} outperforms narrower scopes like per-rollout or per-prompt (Table~\ref{tab:selection_scope}). This allows for a more efficient global allocation of supervision, enabling tokens to compete for guidance across different prompts and responses. The model can thus focus the limited supervision budget on the most complementary updates available in the entire sub-batch, rather than being constrained by a locally optimal but globally suboptimal choice.

\section{Conclusion}
We introduced GIFT, a method that resolves sparse RL and dense distillation conflicts by filtering teacher guidance via token-level gradient geometry, decomposing distillation gradients into components parallel (conflict) and orthogonal (novelty) to the RL gradient for principled scoring. Paired with coverage annealing, this selectively withdraws conflicting updates, smoothly transitioning to a pure RL objective. On complex reasoning benchmarks, GIFT consistently outperforms strong baselines, validating this integration of dense supervision.



\subsection*{AI use statement}
During the preparation of this manuscript, large language models (LLMs) were utilized solely to refine grammar, readability, and formatting. The core scientific contributions—encompassing the research design, methodology, data analysis, and conclusions—were entirely conceptualized and executed by the human authors. We have carefully verified all AI-assisted language refinements and assume full responsibility for the manuscript's final content and integrity.

\subsection*{Reproducibility statement}
To ensure our findings are fully reproducible, all source code and trained model weights will be made publicly available upon the acceptance of this manuscript. Furthermore, we will provide comprehensive documentation and evaluation protocols to assist the research community in replicating our experiments.





\bibliography{iclr2027_conference}

@article{grpo,
  title={Deepseekmath: Pushing the limits of mathematical reasoning in open language models},
  author={Shao, Zhihong and Wang, Peiyi and Zhu, Qihao and Xu, Runxin and Song, Junxiao and Bi, Xiao and Zhang, Haowei and Zhang, Mingchuan and Li, YK and Wu, Yang and others},
  journal={arXiv preprint arXiv:2402.03300},
  year={2024}
}

@article{drgrpo,
  title={Understanding r1-zero-like training: A critical perspective},
  author={Liu, Zichen and Chen, Changyu and Li, Wenjun and Qi, Penghui and Pang, Tianyu and Du, Chao and Lee, Wee Sun and Lin, Min},
  journal={arXiv preprint arXiv:2503.20783},
  year={2025}
}

@article{dapo,
  title={DAPO: An Open-Source LLM Reinforcement Learning System at Scale}, 
      author={Qiying Yu and Zheng Zhang and Ruofei Zhu and Yufeng Yuan and Xiaochen Zuo and Yu Yue and Tiantian Fan and Gaohong Liu and Lingjun Liu and Xin Liu and Haibin Lin and Zhiqi Lin and Bole Ma and Guangming Sheng and Yuxuan Tong and Chi Zhang and Mofan Zhang and Wang Zhang and Hang Zhu and Jinhua Zhu and Jiaze Chen and Jiangjie Chen and Chengyi Wang and Hongli Yu and Weinan Dai and Yuxuan Song and Xiangpeng Wei and Hao Zhou and Jingjing Liu and Wei-Ying Ma and Ya-Qin Zhang and Lin Yan and Mu Qiao and Yonghui Wu and Mingxuan Wang},
      year={2025},
      eprint={2503.14476},
      archivePrefix={arXiv},
      primaryClass={cs.LG},
      url={https://arxiv.org/abs/2503.14476}, 
}

@article{opsd,
  title={Self-distilled reasoner: On-policy self-distillation for large language models},
  author={Zhao, Siyan and Xie, Zhihui and Liu, Mengchen and Huang, Jing and Pang, Guan and Chen, Feiyu and Grover, Aditya},
  journal={arXiv preprint arXiv:2601.18734},
  year={2026}
}

@inproceedings{gkd,
  title={On-policy distillation of language models: Learning from self-generated mistakes},
  author={Agarwal, Rishabh and Vieillard, Nino and Zhou, Yongchao and Stanczyk, Piotr and Ramos Garea, Sabela and Geist, Matthieu and Bachem, Olivier},
  booktitle={International Conference on Learning Representations},
  volume={2024},
  pages={21246--21263},
  year={2024}
}

@article{rlsd,
  title={Self-distilled rlvr},
  author={Yang, Chenxu and Qin, Chuanyu and Si, Qingyi and Chen, Minghui and Gu, Naibin and Yao, Dingyu and Lin, Zheng and Wang, Weiping and Wang, Jiaqi and Duan, Nan},
  journal={arXiv preprint arXiv:2604.03128},
  year={2026}
}

@article{rlrt,
  title={Rebellious student: Reversing teacher signals for reasoning exploration with self-distilled rlvr},
  author={Kim, Jeonghye and Jeon, Jiwon and Li, Dongsheng and Yang, Yuqing},
  journal={arXiv preprint arXiv:2605.10781},
  year={2026}
}

@article{sdpg,
  title={Self-distilled policy gradient},
  author={Liu, Yifeng and Zhang, Shiyuan and Zhang, Yifan and Gu, Quanquan},
  journal={arXiv preprint arXiv:2606.04036},
  year={2026}
}

@article{luffy,
  title={Learning to reason under off-policy guidance, 2025},
  author={Yan, Jianhao and Li, Yafu and Hu, Zican and Wang, Zhi and Cui, Ganqu and Qu, Xiaoye and Cheng, Yu and Zhang, Yue},
  journal={URL https://arxiv. org/abs/2504.14945},
  year={2025}
}

@article{eopd,
  title={Entropy-aware on-policy distillation of language models},
  author={Jin, Woogyeol and Min, Taywon and Yang, Yongjin and Wei, Dennis and Zhou, Yi and Kadhe, Swanand Ravindra and Baracaldo, Nathalie and Lee, Kimin},
  journal={arXiv preprint arXiv:2603.07079},
  year={2026}
}

@article{tip,
  title={Tip: Token importance in on-policy distillation},
  author={Xu, Yuanda and Sang, Hejian and Zhou, Zhengze and He, Ran and Wang, Zhipeng and Geramifard, Alborz},
  journal={arXiv preprint arXiv:2604.14084},
  year={2026}
}

@article{minillm,
  title={Minillm: On-policy distillation of large language models},
  author={Gu, Yuxian and Dong, Li and Wei, Furu and Huang, Minlie},
  journal={arXiv preprint arXiv:2306.08543},
  year={2023}
}

@article{opdvr,
  title={On-policy Distillation with Verifiable Reward},
  author={Lin, Wenze and Zhao, Jiale and Jiang, Xitai and Rao, Songde and Li, Yining and Wang, Shenzhi and He, Bingxiang and Huang, Gao},
  journal={arXiv preprint arXiv:2608.24696},
  year={2026}
}

@article{crpo,
  title={Contrastive Reinforced Policy Optimization via Privileged Self-Distillation}, 
      author={Xingjian Wu and Junlin Liu and Xingchen Liu and Xuhang Zhu and Jianing Wang and Linsen Guo and Xiaoyu Li and Xuezhi Cao and Xunliang Cai},
      year={2026},
      eprint={2607.28026},
      archivePrefix={arXiv},
      primaryClass={cs.LG},
      url={https://arxiv.org/abs/2607.28026}, 
}

@misc{mathai_amc23,
  author       = {{math-ai}},
  title        = {amc23: Dataset compiled from the 2023 American Mathematics Competitions},
  year         = {2023},
  publisher    = {Hugging Face},
  howpublished = {\url{https://huggingface.co/datasets/math-ai/amc23}}
}

@misc{aime24,
  title     = {American Invitational Mathematics Examination ({AIME}) 2024},
  author    = {Zhang, Yifan and {{Math-AI Team}}},
  year      = {2024},
  publisher = {Hugging Face},
  url       = {https://huggingface.co/datasets/math-ai/aime24}
}

@misc{aime25,
  title     = {American Invitational Mathematics Examination ({AIME}) 2025},
  author    = {Zhang, Yifan and {{Math-AI Team}}},
  year      = {2025},
  publisher = {Hugging Face},
  url       = {https://huggingface.co/datasets/math-ai/aime25}
}

@misc{xie2025logicrlunleashingllmreasoning,
      title={Logic-RL: Unleashing LLM Reasoning with Rule-Based Reinforcement Learning}, 
      author={Tian Xie and Zitian Gao and Qingnan Ren and Haoming Luo and Yuqian Hong and Bryan Dai and Joey Zhou and Kai Qiu and Zhirong Wu and Chong Luo},
      year={2025},
      eprint={2502.14768},
      archivePrefix={arXiv},
      primaryClass={cs.CL},
      url={https://arxiv.org/abs/2502.14768}, 
}

@misc{dstg,
      title={Distill Where You Fail: Recovering Learning Signals of Negative RL-Groups from Adaptive Teacher Guidance}, 
      author={Zhuowen Han and Jinwei Xiao and Zhengxi Lu and Renren Jin and Zhiyuan Yao and Yuxin Liu and Hongyan Hao and Yueqing Sun and Yu Yang and Qi GU and Xunliang Cai and Deyi Xiong},
      year={2026},
      eprint={2608.00782},
      archivePrefix={arXiv},
      primaryClass={cs.CL},
      url={https://arxiv.org/abs/2608.00782}, 
}

@misc{SKALD,
      title={Distill Skills into Weights, Not Prompts: Abstract Skills as Privileged Signals for On-Policy Self-Distillation}, 
      author={Yubo Jiang and Fengying Xie and Zhiguo Jiang and Haopeng Zhang},
      year={2026},
      eprint={2608.09826},
      archivePrefix={arXiv},
      primaryClass={cs.LG},
      url={https://arxiv.org/abs/2608.09826}, 
}

@article{UFT,
author       = {Liu, Mingyang and Farina, Gabriele and Ozdaglar, Asuman},
title        = {UFT: Unifying Supervised and Reinforcement Fine-Tuning},
journal      = {arXiv preprint arXiv:2505.16984},
year         = {2025}
}

@article{Schmitt2018KickstartingDR,
  title={Kickstarting Deep Reinforcement Learning},
  author={Simon Schmitt and Jonathan J. Hudson and Augustin Ž{\'i}dek and Simon Osindero and Carl Doersch and Wojciech M. Czarnecki and Joel Z. Leibo and Heinrich K{\"u}ttler and Andrew Zisserman and Karen Simonyan and S. M. Ali Eslami},
  journal={ArXiv},
  year={2018},
  volume={abs/1803.03835},
  url={https://api.semanticscholar.org/CorpusID:3883991}
}

@article{yu2020gradient,
  title={Gradient surgery for multi-task learning},
  author={Yu, Tianhe and Kumar, Saurabh and Gupta, Abhishek and Levine, Sergey and Hausman, Karol and Finn, Chelsea},
  journal={arXiv preprint arXiv:2001.06782},
  year={2020}
}

@inproceedings{
	dery2021auxiliary,
	title={{\{}AUXILIARY{\}} {\{}TASK{\}} {\{}UPDATE{\}} {\{}DECOMPOSITION{\}}: {\{}THE{\}} {\{}GOOD{\}}, {\{}THE{\}} {\{}BAD{\}} {\{}AND{\}} {\{}THE{\}} {\{}NEUTRAL{\}}},
	author={Lucio M. Dery and Yann Dauphin and David Grangier},
	booktitle={International Conference on Learning Representations},
	year={2021},
	url={https://openreview.net/forum?id=1GTma8HwlYp}
}

@article{pan2026rlcsd,
  title={RLCSD: Reinforcement Learning with Contrastive On-Policy Self-Distillation},
  author={Pan, Leyi and Tao, Shuchang and Zhai, Yunpeng and Zhang, Lingzhe and Liu, Zhaoyang and Ding, Bolin and Liu, Aiwei and Wen, Lijie},
  journal={arXiv preprint arXiv:2606.11709},
  year={2026}
}

@misc{wei2026amrsdasymmetricmetareflectiveselfdistillation,
      title={AMR-SD: Asymmetric Meta-Reflective Self-Distillation for Token-Level Credit Assignment}, 
      author={Zhenlin Wei and Pu Jian and Yingzhuo Deng and Xiaohan Wang and Jiajun Chai and Zhexin Hu and Wei Lin and Shanbin Zhang and Guojun Yin},
      year={2026},
      eprint={2605.18529},
      archivePrefix={arXiv},
      primaryClass={cs.AI},
      url={https://arxiv.org/abs/2605.18529}, 
}

@misc{li2026sequentialbeatsjointinterplay,
      title={Sequential Beats Joint: On the Interplay between On-Policy Distillation and RLVR}, 
      author={Boyan Li and Bingsen Chen and Chenghao Yang and Ping Nie and Chen Zhao and Xi Ye},
      year={2026},
      eprint={2609.04108},
      archivePrefix={arXiv},
      primaryClass={cs.CL},
      url={https://arxiv.org/abs/2609.04108}, 
}

@misc{yang2025qwen3technicalreport,
      title={Qwen3 Technical Report}, 
      author={An Yang and Anfeng Li and Baosong Yang and Beichen Zhang and Binyuan Hui and Bo Zheng and Bowen Yu and Chang Gao and Chengen Huang and Chenxu Lv and Chujie Zheng and Dayiheng Liu and Fan Zhou and Fei Huang and Feng Hu and Hao Ge and Haoran Wei and Huan Lin and Jialong Tang and Jian Yang and Jianhong Tu and Jianwei Zhang and Jianxin Yang and Jiaxi Yang and Jing Zhou and Jingren Zhou and Junyang Lin and Kai Dang and Keqin Bao and Kexin Yang and Le Yu and Lianghao Deng and Mei Li and Mingfeng Xue and Mingze Li and Pei Zhang and Peng Wang and Qin Zhu and Rui Men and Ruize Gao and Shixuan Liu and Shuang Luo and Tianhao Li and Tianyi Tang and Wenbiao Yin and Xingzhang Ren and Xinyu Wang and Xinyu Zhang and Xuancheng Ren and Yang Fan and Yang Su and Yichang Zhang and Yinger Zhang and Yu Wan and Yuqiong Liu and Zekun Wang and Zeyu Cui and Zhenru Zhang and Zhipeng Zhou and Zihan Qiu},
      year={2025},
      eprint={2505.09388},
      archivePrefix={arXiv},
      primaryClass={cs.CL},
      url={https://arxiv.org/abs/2505.09388}, 
}

@misc{olmo2026olmo3,
      title={Olmo 3}, 
      author={Team Olmo and : and Allyson Ettinger and Amanda Bertsch and Bailey Kuehl and David Graham and David Heineman and Dirk Groeneveld and Faeze Brahman and Finbarr Timbers and Hamish Ivison and Jacob Morrison and Jake Poznanski and Kyle Lo and Luca Soldaini and Matt Jordan and Mayee Chen and Michael Noukhovitch and Nathan Lambert and Pete Walsh and Pradeep Dasigi and Robert Berry and Saumya Malik and Saurabh Shah and Scott Geng and Shane Arora and Shashank Gupta and Taira Anderson and Teng Xiao and Tyler Murray and Tyler Romero and Victoria Graf and Akari Asai and Akshita Bhagia and Alexander Wettig and Alisa Liu and Aman Rangapur and Chloe Anastasiades and Costa Huang and Dustin Schwenk and Harsh Trivedi and Ian Magnusson and Jaron Lochner and Jiacheng Liu and Lester James V. Miranda and Maarten Sap and Malia Morgan and Michael Schmitz and Michal Guerquin and Michael Wilson and Regan Huff and Ronan Le Bras and Rui Xin and Rulin Shao and Sam Skjonsberg and Shannon Zejiang Shen and Shuyue Stella Li and Tucker Wilde and Valentina Pyatkin and Will Merrill and Yapei Chang and Yuling Gu and Zhiyuan Zeng and Ashish Sabharwal and Luke Zettlemoyer and Pang Wei Koh and Ali Farhadi and Noah A. Smith and Hannaneh Hajishirzi},
      year={2026},
      eprint={2512.13961},
      archivePrefix={arXiv},
      primaryClass={cs.CL},
      url={https://arxiv.org/abs/2512.13961}, 
}

@misc{he2025deepmath103klargescalechallengingdecontaminated,
      title={DeepMath-103K: A Large-Scale, Challenging, Decontaminated, and Verifiable Mathematical Dataset for Advancing Reasoning}, 
      author={Zhiwei He and Tian Liang and Jiahao Xu and Qiuzhi Liu and Xingyu Chen and Yue Wang and Linfeng Song and Dian Yu and Zhenwen Liang and Wenxuan Wang and Zhuosheng Zhang and Rui Wang and Zhaopeng Tu and Haitao Mi and Dong Yu},
      year={2025},
      eprint={2504.11456},
      archivePrefix={arXiv},
      primaryClass={cs.CL},
      url={https://arxiv.org/abs/2504.11456}, 
}

@misc{yan2025learningreasonoffpolicyguidance,
      title={Learning to Reason under Off-Policy Guidance}, 
      author={Jianhao Yan and Yafu Li and Zican Hu and Zhi Wang and Ganqu Cui and Xiaoye Qu and Yu Cheng and Yue Zhang},
      year={2025},
      eprint={2504.14945},
      archivePrefix={arXiv},
      primaryClass={cs.LG},
      url={https://arxiv.org/abs/2504.14945}, 
}

@misc{sdpo,
      title={Reinforcement Learning via Self-Distillation}, 
      author={Jonas Hübotter and Frederike Lübeck and Lejs Behric and Anton Baumann and Marco Bagatella and Daniel Marta and Ido Hakimi and Idan Shenfeld and Thomas Kleine Buening and Carlos Guestrin and Andreas Krause},
      year={2026},
      eprint={2601.20802},
      archivePrefix={arXiv},
      primaryClass={cs.LG},
      url={https://arxiv.org/abs/2601.20802}, 
}

@article{Guo_2025,
   title={DeepSeek-R1 incentivizes reasoning in LLMs through reinforcement learning},
   volume={645},
   ISSN={1476-4687},
   url={http://dx.doi.org/10.1038/s41586-025-09422-z},
   DOI={10.1038/s41586-025-09422-z},
   number={8081},
   journal={Nature},
   publisher={Springer Science and Business Media LLC},
   author={Guo, Daya and Yang, Dejian and Zhang, Haowei and Song, Junxiao and Wang, Peiyi and Zhu, Qihao and Xu, Runxin and Zhang, Ruoyu and Ma, Shirong and Bi, Xiao and Zhang, Xiaokang and Yu, Xingkai and Wu, Yu and Wu, Z. F. and Gou, Zhibin and Shao, Zhihong and Li, Zhuoshu and Gao, Ziyi and Liu, Aixin and Xue, Bing and Wang, Bingxuan and Wu, Bochao and Feng, Bei and Lu, Chengda and Zhao, Chenggang and Deng, Chengqi and Ruan, Chong and Dai, Damai and Chen, Deli and Ji, Dongjie and Li, Erhang and Lin, Fangyun and Dai, Fucong and Luo, Fuli and Hao, Guangbo and Chen, Guanting and Li, Guowei and Zhang, H. and Xu, Hanwei and Ding, Honghui and Gao, Huazuo and Qu, Hui and Li, Hui and Guo, Jianzhong and Li, Jiashi and Chen, Jingchang and Yuan, Jingyang and Tu, Jinhao and Qiu, Junjie and Li, Junlong and Cai, J. L. and Ni, Jiaqi and Liang, Jian and Chen, Jin and Dong, Kai and Hu, Kai and You, Kaichao and Gao, Kaige and Guan, Kang and Huang, Kexin and Yu, Kuai and Wang, Lean and Zhang, Lecong and Zhao, Liang and Wang, Litong and Zhang, Liyue and Xu, Lei and Xia, Leyi and Zhang, Mingchuan and Zhang, Minghua and Tang, Minghui and Zhou, Mingxu and Li, Meng and Wang, Miaojun and Li, Mingming and Tian, Ning and Huang, Panpan and Zhang, Peng and Wang, Qiancheng and Chen, Qinyu and Du, Qiushi and Ge, Ruiqi and Zhang, Ruisong and Pan, Ruizhe and Wang, Runji and Chen, R. J. and Jin, R. L. and Chen, Ruyi and Lu, Shanghao and Zhou, Shangyan and Chen, Shanhuang and Ye, Shengfeng and Wang, Shiyu and Yu, Shuiping and Zhou, Shunfeng and Pan, Shuting and Li, S. S. and Zhou, Shuang and Wu, Shaoqing and Yun, Tao and Pei, Tian and Sun, Tianyu and Wang, T. and Zeng, Wangding and Liu, Wen and Liang, Wenfeng and Gao, Wenjun and Yu, Wenqin and Zhang, Wentao and Xiao, W. L. and An, Wei and Liu, Xiaodong and Wang, Xiaohan and Chen, Xiaokang and Nie, Xiaotao and Cheng, Xin and Liu, Xin and Xie, Xin and Liu, Xingchao and Yang, Xinyu and Li, Xinyuan and Su, Xuecheng and Lin, Xuheng and Li, X. Q. and Jin, Xiangyue and Shen, Xiaojin and Chen, Xiaosha and Sun, Xiaowen and Wang, Xiaoxiang and Song, Xinnan and Zhou, Xinyi and Wang, Xianzu and Shan, Xinxia and Li, Y. K. and Wang, Y. Q. and Wei, Y. X. and Zhang, Yang and Xu, Yanhong and Li, Yao and Zhao, Yao and Sun, Yaofeng and Wang, Yaohui and Yu, Yi and Zhang, Yichao and Shi, Yifan and Xiong, Yiliang and He, Ying and Piao, Yishi and Wang, Yisong and Tan, Yixuan and Ma, Yiyang and Liu, Yiyuan and Guo, Yongqiang and Ou, Yuan and Wang, Yuduan and Gong, Yue and Zou, Yuheng and He, Yujia and Xiong, Yunfan and Luo, Yuxiang and You, Yuxiang and Liu, Yuxuan and Zhou, Yuyang and Zhu, Y. X. and Huang, Yanping and Li, Yaohui and Zheng, Yi and Zhu, Yuchen and Ma, Yunxian and Tang, Ying and Zha, Yukun and Yan, Yuting and Ren, Z. Z. and Ren, Zehui and Sha, Zhangli and Fu, Zhe and Xu, Zhean and Xie, Zhenda and Zhang, Zhengyan and Hao, Zhewen and Ma, Zhicheng and Yan, Zhigang and Wu, Zhiyu and Gu, Zihui and Zhu, Zijia and Liu, Zijun and Li, Zilin and Xie, Ziwei and Song, Ziyang and Pan, Zizheng and Huang, Zhen and Xu, Zhipeng and Zhang, Zhongyu and Zhang, Zhen},
   year={2025},
   pages={633–638} }

@misc{yin2018diversity,
      title={Gradient Diversity: a Key Ingredient for Scalable Distributed Learning}, 
      author={Dong Yin and Ashwin Pananjady and Max Lam and Dimitris Papailiopoulos and Kannan Ramchandran and Peter Bartlett},
      year={2018},
      eprint={1706.05699},
      archivePrefix={arXiv},
      primaryClass={cs.LG},
      url={https://arxiv.org/abs/1706.05699}, 
}
\bibliographystyle{iclr2027_conference}

\newpage

\appendix

\section{Main Notation}
\label{sec:main_notation}

For clarity, Table~\ref{tab:main_notation} summarizes the main notations
used in GIFT.

\begin{table*}[htbp]
\centering
\caption{Summary of the main notations used in GIFT.}
\label{tab:main_notation}
\renewcommand{\arraystretch}{1.12}
\setlength{\tabcolsep}{7pt}
\resizebox{0.99\textwidth}{!}{
\begin{tabular}{l l p{9.8cm}}
\toprule
\textbf{Notation} & \textbf{Type} & \textbf{Meaning} \\
\midrule

$i,\mathcal{I}$ 
& index / set 
& Token position $i$ and the set $\mathcal{I}$ of all valid token positions. \\

$z_i,\ p_i$ 
& vectors 
& Student logits and predictive distribution, where
$p_i=\operatorname{softmax}(z_i)$. \\

$\tilde q_i$ 
& distribution 
& Privileged teacher distribution at position $i$. \\

$u_i,\ e_{u_i}$ 
& token / vector 
& Sampled token and its corresponding one-hot vector. \\

$w_i,\ A_i,\ \rho_i,\ h_i$ 
& scalars 
& Token weight, advantage, importance ratio, and active-branch indicator in the clipped GRPO objective. \\

$\ell_i^{\mathrm{GRPO}}$ 
& loss 
& Token-level clipped GRPO loss providing outcome-grounded supervision. \\

$\ell_i^{\mathrm{sKL}}$ 
& loss 
& Token-level self-distillation loss aligning $p_i$ with $\tilde q_i$. \\

$r_i$ 
& vector 
& Outcome-driven logit gradient,
$r_i=\nabla_{z_i}\ell_i^{\mathrm{GRPO}}$. \\

$d_i$ 
& vector 
& Imitation-driven logit gradient,
$d_i=\nabla_{z_i}(w_i\ell_i^{\mathrm{sKL}})$. \\

$a_i$ 
& scalar 
& Signed parallel component,
$a_i=\langle d_i,\hat r_i\rangle$, measuring assistance or conflict with the RL gradient. \\

$b_i$ 
& scalar 
& Orthogonal magnitude,
$b_i=\lVert d_i-a_i\hat r_i\rVert_2$, measuring complementary teacher guidance. \\

$\mathcal{D}$ 
& set 
& Token positions with a usable non-zero RL reference gradient. \\

$\mathcal{E}_m$ 
& set 
& Distillation-eligible token positions in training sub-batch $m$. \\

$\mathcal{C}_m,\mathcal{U}_m$ 
& sets 
& Conflicting and non-conflicting eligible tokens, respectively. \\

$s_i$ 
& scalar 
& GIFT score: $a_i$ for conflicting tokens, $b_i$ for non-conflicting tokens, and zero when no usable RL reference exists. \\

$k_{t,m}$ 
& integer 
& Number of tokens retained in training sub-batch $m$ at step $t$. \\

$m_i(t)$ 
& indicator 
& Binary mask indicating whether token $i$ is selected at step $t$. \\

$\kappa_t$ 
& scalar 
& Retention ratio controlling the progressively shrinking token budget. \\

$H,\ T_{\mathrm{exit}}$ 
& integers 
& End of the full-coverage phase and step of complete teacher withdrawal. \\

$\lambda$ 
& scalar 
& Distillation coefficient, fixed at $\lambda=1$. \\

$N_{\mathrm{valid}}$ 
& integer 
& Total number of valid token positions used for loss normalization. \\

\bottomrule
\end{tabular}
}
\end{table*}

\section{Supporting Lemmas}
\label{app:lemmas}

Throughout we fix a token position and drop the index where unambiguous. Let
$p=\mathrm{softmax}(z)$, let $S$ be the teacher support with $|S|=M$, and set
\begin{equation}
  P_{\mathrm{tail}}=1-\sum_{v\in S}p(v)>0,\qquad
  Q_{\mathrm{tail}}=1-\sum_{v\in S}q(v)>0,\qquad
  \sigma=\frac{Q_{\mathrm{tail}}}{P_{\mathrm{tail}}},
\end{equation}
so the objective is finite and differentiable. We use
$\partial p(v)/\partial z(k)=p(v)(\delta_{vk}-p(k))$.

\begin{lemma}[Sparse forward-KL gradient]
\label{lem:skl}
The loss
$\ell^{\mathrm{sKL}}=\sum_{v\in S}q(v)\log\frac{q(v)}{p(v)}
+Q_{\mathrm{tail}}\log\frac{Q_{\mathrm{tail}}}{P_{\mathrm{tail}}}$
satisfies $\nabla_z\ell^{\mathrm{sKL}}=p-\tilde q$, where the tail-completed
teacher
\begin{equation}
  \tilde q(v)=q(v)\ \ (v\in S),\qquad
  \tilde q(v)=\sigma\,p(v)\ \ (v\notin S)
  \label{eq:qtilde}
\end{equation}
is a probability distribution.
\end{lemma}

\begin{proof}
Only $-\sum_{v\in S}q(v)\log p(v)$ and $-Q_{\mathrm{tail}}\log P_{\mathrm{tail}}$
depend on $z$, as $q$ and $Q_{\mathrm{tail}}$ are detached. For the first,
\begin{equation}
  \frac{\partial}{\partial z(k)}\Bigl(-\sum_{v\in S}q(v)\log p(v)\Bigr)
  =-\sum_{v\in S}q(v)\bigl(\delta_{vk}-p(k)\bigr)
  =-q_S(k)+p(k)\sum_{v\in S}q(v),
\end{equation}
with $q_S(k)=q(k)\mathbf{1}\{k\in S\}$. For the second, since
$P_{\mathrm{tail}}=\sum_{v\notin S}p(v)$,
\begin{equation}
  \frac{\partial P_{\mathrm{tail}}}{\partial z(k)}
  =\mathbf{1}\{k\notin S\}p(k)-p(k)P_{\mathrm{tail}},
\end{equation}
so
$\frac{\partial}{\partial z(k)}(-Q_{\mathrm{tail}}\log P_{\mathrm{tail}})
=-\sigma\mathbf{1}\{k\notin S\}p(k)+p(k)Q_{\mathrm{tail}}$.
Adding the two and using $\sum_{v\in S}q(v)+Q_{\mathrm{tail}}=1$ gives
$p(k)-\tilde q(k)$. Finally $\tilde q\ge0$ and
$\sum_v\tilde q(v)=\sum_{v\in S}q(v)+\sigma P_{\mathrm{tail}}
=\sum_{v\in S}q(v)+Q_{\mathrm{tail}}=1$.
\end{proof}

\begin{proposition}[Closed-form scoring]
\label{prop:closedform}
With $\Pi=\sum_{v\in V}p(v)^2$ and $\pi_S=\sum_{v\in S}p(v)^2$,
\begin{equation}
  \|p-e_u\|_2^2=\Pi-2p(u)+1,\qquad
  \langle p,\tilde q\rangle=\!\!\sum_{v\in S}\!p(v)q(v)+\sigma(\Pi-\pi_S),\qquad
  \|\tilde q\|_2^2=\!\!\sum_{v\in S}\!q(v)^2+\sigma^2(\Pi-\pi_S),
\end{equation}
and $\langle p-\tilde q,\,p-e_u\rangle=\Pi-p(u)-\langle p,\tilde q\rangle+\tilde q(u)$.
Consequently $a$ and $b$ are computable from $\Pi$, $p(u)$, $\tilde q(u)$ and
$O(M)$ sums over $S$, with $\|d\|_2^2=w^2(\Pi-2\langle p,\tilde q\rangle
+\|\tilde q\|_2^2)$ and $b^2=\|d\|_2^2-a^2$. No $|V|$-dimensional gradient is
materialized and no auxiliary backward pass is needed.
\end{proposition}

\begin{proof}
Each identity splits the sum over $S$ and its complement and substitutes
Equation~\ref{eq:qtilde}: e.g.
$\langle p,\tilde q\rangle=\sum_{v\in S}p(v)q(v)+\sigma\sum_{v\notin S}p(v)^2$
and $\sum_{v\notin S}p(v)^2=\Pi-\pi_S$; identically for $\|\tilde q\|_2^2$ with
$\sigma^2$. For the last, expand
$\langle p-\tilde q,p-e_u\rangle=\|p\|_2^2-\langle p,e_u\rangle
-\langle\tilde q,p\rangle+\langle\tilde q,e_u\rangle$. The only vocabulary-wide
quantity is $\Pi$, a scalar reduction of probabilities already formed in the
forward pass. The expression for $b$ follows from $a^2+b^2=\|d\|_2^2$
(orthogonality).
\end{proof}

\emph{Implementation note.} $\tilde q(u)=q(u)$ if $u\in S$, else $\sigma p(u)$;
$b^2$ is clipped at zero before the square root to absorb cancellation when
$a^2\approx\|d\|_2^2$.

\begin{proposition}[Invariance of the score]
\label{prop:invariance}
For $i\in\mathcal{D}$,
\begin{equation}
  a_i=w_i\operatorname{sign}(A_i)\,
  \frac{\langle p_i-\tilde q_i,\;p_i-e_{u_i}\rangle}{\|p_i-e_{u_i}\|_2},
  \qquad b_i^2=\|d_i\|_2^2-a_i^2,
\end{equation}
so $a_i$ and $b_i$ are invariant to $|A_i|$ and to positive rescaling of
$\rho_i$, provided the active-branch and norm-threshold decisions are unchanged;
$b_i$ is in addition independent of $\operatorname{sign}(A_i)$.
\end{proposition}

\begin{proof}
On the active branch $r_i=c\psi$ with $c=w_iA_i\rho_i\neq0$ and
$\psi=p_i-e_{u_i}\neq0$. Since $w_i,\rho_i>0$ we have $|c|=w_i|A_i|\rho_i$ and
\begin{equation}
  \hat r_i=\frac{c\psi}{|c|\|\psi\|_2}
  =\operatorname{sign}(A_i)\frac{\psi}{\|\psi\|_2},
\end{equation}
so the magnitude cancels. With $d_i=w_i(p_i-\tilde q_i)$ from
Lemma~\ref{lem:skl}, $a_i=\langle d_i,\hat r_i\rangle$ gives the display.
$\|d_i\|_2$ involves neither $A_i$ nor $\rho_i$, and $b_i^2=\|d_i\|_2^2-a_i^2$
depends on $a_i$ only through $a_i^2$, which is insensitive to
$\operatorname{sign}(A_i)$.
\end{proof}

\subsection{Gradient Diversity}
\label{app:diversity}

Theorem~\ref{thm:diversity}(iii) uses the following quantity, introduced by
\citet{yin2018diversity} to characterize how far mini-batch SGD can be scaled
before departing from the behaviour of its serial counterpart. We state it and
the two facts we rely on, so that the argument is self-contained.

\begin{definition}[Gradient diversity]
\label{def:diversity}
For gradients $g_1,\dots,g_n$ summed into a single update with
$\sum_jg_j\neq0$,
\begin{equation}
  \Delta\bigl(\{g_j\}\bigr)
  :=\frac{\sum_{j=1}^n\|g_j\|_2^{2}}{\bigl\|\sum_{j=1}^n g_j\bigr\|_2^{2}}
  =\frac{\sum_j\|g_j\|_2^2}
        {\sum_j\|g_j\|_2^2+\sum_{j\neq l}\langle g_j,g_l\rangle}.
  \label{eq:diversitydef}
\end{equation}
\end{definition}

\begin{lemma}[Endpoints of $\Delta$]
\label{lem:diversity}
If $g_1=\dots=g_n=g\neq0$ then $\Delta=1/n$. If the $g_j$ are nonzero and
mutually orthogonal then $\Delta=1$, whatever their norms. In general
$\Delta>1$ if and only if $\sum_{j\neq l}\langle g_j,g_l\rangle<0$.
\end{lemma}

\begin{proof}
For identical gradients the numerator is $n\|g\|_2^2$ and the denominator
$\|ng\|_2^2=n^2\|g\|_2^2$, giving $1/n$. For mutually orthogonal gradients every
cross term in the second form of Equation~\ref{eq:diversitydef} vanishes, leaving
$\Delta=1$. The final claim is immediate from that same form, since the
numerator is positive.
\end{proof}

So $\Delta$ is a redundancy measure: it is smallest ($1/n$) when the summands
duplicate one another, equals exactly $1$ when they are mutually
non-interfering, and exceeds $1$ only when they partially cancel---the last case
being why an unconstrained preference for large $\Delta$ is unsafe. The regime
relevant here is $n=2$ with $g_1=r_i$ and $g_2=\lambda d_i$. We use only
Definition~\ref{def:diversity} and Lemma~\ref{lem:diversity}; the convergence
and stability bounds of \citet{yin2018diversity} are cited as context for why
low redundancy is desirable, but no statement in
Theorem~\ref{thm:diversity}(iii) depends on them.

\section{Proofs of the Main Results}
\label{app:proofs}

\subsection{Theorem~\ref{thm:noninterference} (non-interference)}

\begin{proof}
Let $\phi(\eta)=f_i(z_i-\eta(r_i+\lambda d_i))$. By the chain rule and
$\nabla f_i(z_i)=r_i$,
\begin{equation}
  \phi'(0)=\langle r_i,\,-(r_i+\lambda d_i)\rangle
  =-\|r_i\|_2^2-\lambda\langle r_i,d_i\rangle
  =-\|r_i\|_2^2-\lambda\|r_i\|_2a_i,
\end{equation}
using $r_i=\|r_i\|_2\hat r_i$ and $a_i=\langle d_i,\hat r_i\rangle$. 
This is $\frac{d}{d\eta}f_{i}(z_{i}(\eta))|_{\eta=0} = -||r_{i}||_{2}^{2} - \lambda||r_{i}||_{2}a_{i}$.
The GRPO-only step is $\lambda=0$, with value
$-\|r_i\|_2^2<0$; the teacher's contribution is the residual
$-\lambda\|r_i\|_2a_i$, whose sign is opposite to that of $a_i$ since
$\lambda\ge0,\|r_i\|_2>0$.

For (1): if $a_i>0$ then $\phi'(0)<-\|r_i\|_2^2$, a strictly larger decrease; if
$a_i=0$ then $\phi'(0)=-\|r_i\|_2^2$ exactly, independently of $\|d_i\|_2$, and
in particular the orthogonal component has no first-order effect at any
magnitude. For $a_i\ge0$ and any $\lambda\ge0$ we get $\phi'(0)\le-\|r_i\|_2^2<0$,
so the primary descent cannot be cancelled.

For (2): if $a_i<0$ then $\phi'(0)=-\|r_i\|_2^2+\lambda\|r_i\|_2|a_i|$, which
exceeds $-\|r_i\|_2^2$ by $\lambda\|r_i\|_2|a_i|>0$ for $\lambda>0$: the decrease
is necessarily smaller. It becomes nonnegative, i.e.\ the primary loss no longer
decreases to first order, exactly when
$\lambda\|r_i\|_2|a_i|\ge\|r_i\|_2^2$, i.e.\ $\lambda|a_i|\ge\|r_i\|_2$;
strict reversal ($\phi'(0)>0$) holds when $\lambda|a_i|>\|r_i\|_2$.
\end{proof}

\subsection{Theorem~\ref{thm:diversity} (complementarity)}

\begin{proof}[Part (i)]
Decompose $d_i=a_i\hat r_i+d_i^{\perp}$. Feasibility $r_i^\top v=0$ is equivalent
to $\hat r_i^\top v=0$, whence $-d_i^\top v=-(d_i^{\perp})^\top v$. By
Cauchy--Schwarz this is at most $\|d_i^{\perp}\|_2\|v\|_2\le b_i$. If $b_i>0$,
$v^\star=-d_i^{\perp}/b_i$ is feasible ($\|v^\star\|_2=1$,
$r_i^\top v^\star=0$) and attains $b_i$. If $b_i=0$ then $d_i^{\perp}=0$ and
every feasible $v$ gives $0=b_i$. (By Theorem~\ref{thm:noninterference} with the
perturbation $z_i+v$, the constraint $r_i^\top v=0$ is precisely first-order
neutrality of the primary loss.)
\end{proof}

\begin{proof}[Part (ii)]
Equip $\mathcal H=\bigoplus_{i\in\mathcal I}\mathbb R^{|\mathcal V|}$
with the Euclidean product norm
$\|(x_i)_i\|_{\mathcal H}^2=\sum_{i\in\mathcal I}\|x_i\|_2^2$.
Since $\mathcal I$ is finite,
$\mathcal M_0=\bigoplus_{i\in\mathcal I}\operatorname{span}(r_i)$
is a closed linear subspace of $\mathcal H$, where
$\operatorname{span}(0)=\{0\}$. Distinct token blocks are orthogonal,
so the orthogonal projection of $D_m=(m_i d_i)_i$ acts blockwise:
\begin{equation}
    (\Pi_{\mathcal M_0}D_m)_i=
    \begin{cases}
        m_i\dfrac{\langle d_i,r_i\rangle}{\|r_i\|_2^2}\,r_i,
            & r_i\ne0,\\[6pt]
        0,  & r_i=0.
    \end{cases}
\end{equation}
By the Pythagorean identity in each block and $m_i^2=m_i$,
\begin{equation}
\begin{aligned}
    \operatorname{dist}^2(D_m,\mathcal M_0)
    &=\|D_m-\Pi_{\mathcal M_0}D_m\|_{\mathcal H}^2\\
    &=\sum_{i:r_i\ne0}m_i
      \left(\|d_i\|_2^2-
      \frac{\langle d_i,r_i\rangle^2}{\|r_i\|_2^2}\right)
      +\sum_{i:r_i=0}m_i\|d_i\|_2^2.
\end{aligned}
\end{equation}
For $i\in\mathcal D_\epsilon=\{i:\|r_i\|_2>\epsilon_g\}$,
where $\epsilon_g\ge0$, the parenthesized term equals $b_i^2$
by the definition of the orthogonal residual. All remaining terms
are nonnegative. Consequently,
\begin{equation}
    \operatorname{dist}^2(D_m,\mathcal M_0)
    \ge\sum_{i\in\mathcal D_\epsilon}m_i b_i^2.
\end{equation}
If $m_i=1$ and $b_i>0$ for some $i\in\mathcal D_\epsilon$,
this lower bound is strictly positive, hence $D_m\notin\mathcal M_0$.
Thus the selected teacher field cannot be represented by token-local
scalar rescalings of the fixed GRPO reference field.
\end{proof}

\paragraph{Scope of Part (ii).}
The comparison space $\mathcal M_0$ permits arbitrary real-valued
rescalings of the fixed token-level GRPO gradients. It therefore
contains the fields produced by scalar-modulation rules whose
per-token gradients remain multiples of the same $r_i$, but need
not equal the reachable set of a particular algorithm with
nonnegative or bounded multipliers. No reference direction is
truncated: a small but nonzero $r_i$ still spans a line, whereas
$r_i=0$ contributes the zero subspace. The threshold
$\epsilon_g$ enters only through the lower bound involving
$\mathcal D_\epsilon$, not through the definition of $\mathcal M_0$.

The result concerns independent token blocks in logit space.
Non-membership in $\mathcal M_0$ does not imply non-membership in
its image after mapping through the shared model Jacobians;
distinct token-level fields can induce the same parameter gradient.
Thus, the theorem establishes local directional irreducibility,
not shared-parameter or function-space novelty, nor improved
verifier performance. Collinear teacher guidance can still affect
update magnitudes and optimization behavior.

Finally, the identity holds for any fixed binary mask and does not
assert that GIFT globally maximizes the distance to $\mathcal M_0$.
Its third lexicographic criterion maximizes only the selected
residual energy on the defined-reference, non-conflicting branch,
after the first two objective coordinates have been optimized.
This restricted selection property is distinct from maximizing
the total residual energy over all tokens.

\begin{proof}[Part (iii)]
Apply Definition~\ref{def:diversity} with $n=2$, $g_1=r_i$ and
$g_2=\lambda d_i$, so $\Delta_i=(\|r_i\|_2^2+\lambda^2\|d_i\|_2^2)/
\|r_i+\lambda d_i\|_2^2$. Write $\nu=\|r_i\|_2>0$ and $\mu=\lambda\|d_i\|_2$,
and expand the denominator using $\langle r_i,d_i\rangle=\nu a_i$:
\begin{equation}
  \|r_i+\lambda d_i\|_2^2=\nu^2+2\lambda\langle r_i,d_i\rangle+\lambda^2\|d_i\|_2^2
  =\nu^2+2\lambda\nu a_i+\mu^2 .
\end{equation}
Hence $\Delta_i=(\nu^2+\mu^2)/(\nu^2+\mu^2+2\lambda\nu a_i)$ has a numerator
independent of $a_i$ and a denominator strictly increasing in $a_i$ (slope
$2\lambda\nu>0$), so $\Delta_i$ is strictly decreasing in $a_i$ wherever the
denominator is positive. On the safe branch $a_i\in[0,\|d_i\|_2]$ we have
$a_i=\sqrt{\|d_i\|_2^2-b_i^2}$, strictly decreasing in $b_i\in[0,\|d_i\|_2]$;
composing the two monotonicities, $\Delta_i$ is strictly increasing in $b_i$
there. At $b_i=0$ (so $a_i=\|d_i\|_2$, the two gradients collinear) the value is
$(\nu^2+\mu^2)/(\nu+\mu)^2$, which lies in $[\tfrac12,1)$ for $\nu,\mu>0$: the
upper bound is strict by AM--GM, and the lower bound follows from
$(\nu+\mu)^2\le2(\nu^2+\mu^2)$ with equality iff $\nu=\mu$. At $a_i=0$
(orthogonal) the value is exactly $1$, consistent with Lemma~\ref{lem:diversity}.

For the unconstrained statement, $\Delta_i$ is strictly decreasing in $a_i$ over
the whole range $a_i\in[-\|d_i\|_2,\|d_i\|_2]$, hence maximized as
$a_i\to-\|d_i\|_2$; there $d_i=-\|d_i\|_2\hat r_i$ and the denominator equals
$(\nu-\mu)^2$, which vanishes when $\mu=\nu$, i.e.\ exactly when
$r_i+\lambda d_i=0$, at which point $\Delta_i$ diverges. By
Lemma~\ref{lem:diversity} this is the regime $\Delta_i>1$, in which the two
gradients partially cancel.
\end{proof}


\subsection{Theorem~\ref{thm:selection} (selection)}

\begin{proof}
Write $\mathcal C=\mathcal C_m$, $\mathcal U=\mathcal U_m$, and $\mathcal E=\mathcal E_m$, with the sets and objectives defined before Equation~\eqref{eq:lexprogram} 
Feasibility means $A\subseteq\mathcal E$ and $|A|=k$, for an integer $0\le k\le|\mathcal E|$.

\emph{Separation.} If $i\in\mathcal C$, then $s_i=a_i<0$. If $j\in\mathcal U\cap\mathcal D$, then $s_j=b_j\ge0$; if $j\in\mathcal U\setminus\mathcal D$, then $s_j=0$. Thus every candidate in $\mathcal U$ precedes every candidate in $\mathcal C$.

\emph{First coordinate.} Every size-$k$ subset satisfies
\[
 N^-(A)=k-|A\cap\mathcal U|\ge\max\{0,k-|\mathcal U|\}.
\]
A top-$k$ set attains this bound: it is contained in $\mathcal U$ when $k\le|\mathcal U|$, and contains all of $\mathcal U$ otherwise. This proves Equation~\eqref{eq:gift-selected-conflicts}.

\emph{Second coordinate.} Let $n=\max\{0,k-|\mathcal U|\}$. If $n=0$, then $C^-=0$ for every first-coordinate minimizer. If $n>0$, every such minimizer contains all of $\mathcal U$ and exactly $n$ elements of $\mathcal C$. Since $s_i=-|a_i|$ on $\mathcal C$, top-$k$ selection chooses the $n$ smallest $|a_i|$, minimizing $C^-$.

\emph{Third coordinate.} If $k>|\mathcal U|$, all first-coordinate minimizers contain $\mathcal U$, so $E^+$ is constant. If $k\le|\mathcal U|$, define $\beta_i=b_i$ on $\mathcal U\cap\mathcal D$ and $\beta_i=0$ on $\mathcal U\setminus\mathcal D$. Then $s_i=\beta_i\ge0$ on $\mathcal U$, and $E^+(A)=\sum_{i\in A}\beta_i^2$ for $A\subseteq\mathcal U$. Because squaring preserves order on $[0,\infty)$, the $k$ largest scores maximize $E^+$.

Equal scores give equal contributions to the relevant objective coordinate, including ties between $b_i=0$ and no-usable-reference positions. Hence arbitrary tie-breaking preserves all three objective values. The cases $k=0$ and $k=|\mathcal E|$ are included, as is the empty pool with $k=0$.
\end{proof}

\paragraph{Remark (Normalized opposition is not primary-loss damage).}
By Theorem~\ref{thm:noninterference}, the first-order change in the primary loss induced by an
admitted conflicting token $i$ is $-\lambda \|r_i\|_2\, a_i > 0$, which
scales with $\|r_i\|_2$ as well as $a_i$. The coordinate $O^{-}(A) =
\sum_{i \in A \cap C_m} |a_i|$ is therefore a \emph{normalized}
opposition---first-order damage per unit of RL-gradient norm (with
$\lambda = 1$)---and serves only as a ranking key. It coincides with
aggregate primary-loss damage only if all admitted conflicting tokens
share the same $\|r_i\|_2$, which is non-generic; we therefore avoid
calling it harm or severity.

\paragraph{Remark (Exact quota vs.\ hard constraint).}
Equation~\ref{eq:lexprogram} imposes $|A| = k$ rather than the
hard-constrained program $\max_m \sum_i m_i b_i^2$ subject to $a_i \ge 0$
for selected $i$ and $\sum_i m_i \le k$, which admits no conflicting token
and may select fewer than $k$ tokens. That hard-constrained rule is
implemented as the \emph{Conflict-Free} ablation (Appendix~\ref{app:gift-ablations},
Table~\ref{tab:selection_strategy}). The two coincide whenever the quota does not exceed the number
of non-conflicting eligible tokens ($k \le |U_m|$, the first case of
Theorem~\ref{thm:selection}); they differ only when the quota forces admission of
conflicting tokens, in which case GIFT admits those with the smallest
normalized opposition. The exact quota $|A_{t,m}| = \min\{K_{t,m},
|E_m|\}$ (Equation~\ref{eq:gift-app-budget}) moreover keeps realized sub-batch coverage a
deterministic function of the schedule $\kappa_t$ and the eligible set,
which the exposure accounting of Appendix~D relies on.

\section{Behaviour at GRPO Stationarity}
\label{app:stationary}

\begin{corollary}[Local behavior when the primary token gradient vanishes]
\label{cor:stationary}
Suppose $r_i=0$ and the token-level primary surrogate is differentiable. Then
every local logit perturbation is first-order neutral to that token-level
surrogate. Under the GIFT objective, the auxiliary logit-gradient block at token
$i$ is
\begin{equation}
    m_i(t)\frac{\lambda w_i}{N_{\mathrm{valid}}}d_i.
\end{equation}
Hence token $i$ contributes a nonzero distillation update only if
$m_i(t)=1$, $w_i\neq0$, $d_i\neq0$, and $t<T_{\mathrm{exit}}$. Eligibility or a
zero score alone does not guarantee selection.
\end{corollary}

\begin{proof}
Because $r_i=\nabla_{z_i}f_i=0$, the directional derivative of $f_i$ along any
local logit perturbation $v$ is $r_i^\top v=0$. The auxiliary term in the GIFT
objective is multiplied by the detached mask $m_i(t)$ and normalized by
$N_{\mathrm{valid}}$, which gives the stated block. When no usable reference
direction exists, the implementation assigns token $i$ score zero; its update
therefore depends on whether it is selected under the exact quota.
\end{proof}

When the sampled GRPO gradient vanishes, GIFT assigns the token a zero score;
whether it receives distillation still depends on the exact-quota selection mask
(Corollary~1).



We record this as a structural observation rather than a claim about aggregate
behaviour. Whether a prompt whose group is uniformly correct or wrong
\emph{should} receive teacher supervision is an empirical question---the
supervision is first-order free with respect to GRPO, but freeness is not the
same as usefulness, and the group-level reward carries no information about the
teacher's reliability there. The corollary establishes only that GIFT's scoring
is well defined in this regime and that it assigns such tokens neutral rather
than negative priority.

\section{Termination and Bounded Exposure}
\label{app:exposure}

\begin{corollary}[Termination and nominal exposure]
\label{cor:exposure}
For $t\ge T_{\mathrm{exit}}$ all masks vanish and $L^{\mathrm{GIFT}}_t$ equals
the native GRPO objective identically, so teacher and scoring passes may be
skipped and the terminal phase is unprivileged. The nominal teacher exposure over
$T\ge T_{\mathrm{exit}}$ actor steps is
\begin{equation}
  \frac{1}{T}\sum_{t=1}^{T}\kappa_t
  =\frac{1}{T}\Bigl(H+\tfrac12\bigl(T_{\mathrm{exit}}-H-1\bigr)\Bigr),
  \label{eq:exposure}
\end{equation}
equal to $43.5/96\approx45.3\%$ for $(H,T_{\mathrm{exit}},T)=(10,78,96)$.
\end{corollary}

\begin{proof}
For $t\ge T_{\mathrm{exit}}$, $\kappa_t=0$ gives
$k_m(t)=\min\{0,|\mathcal{E}_m|\}=0$, so $\mathrm{TopK}$ returns $\emptyset$ and
$m_i(t)=0$ for all $i$; the auxiliary sum in Equation~\ref{eq:objective} vanishes
identically and $L^{\mathrm{GIFT}}_t=L^{\mathrm{GRPO}}_t$. Since the teacher and
scoring computations enter the loss only through that sum, both may be skipped
without changing the objective or its gradient. For the exposure, the first $H$
steps contribute $H$, and substituting $j=T_{\mathrm{exit}}-t$,
\begin{equation}
  \sum_{t=H+1}^{T_{\mathrm{exit}}-1}\frac{T_{\mathrm{exit}}-t}{T_{\mathrm{exit}}-H}
  =\frac{1}{T_{\mathrm{exit}}-H}\sum_{j=1}^{T_{\mathrm{exit}}-H-1}j
  =\frac{T_{\mathrm{exit}}-H-1}{2},
\end{equation}
while steps $t\ge T_{\mathrm{exit}}$ contribute zero. Dividing by $T$ gives
Equation~\ref{eq:exposure}; with $(H,T_{\mathrm{exit}},T)=(10,78,96)$ it is
$(10+33.5)/96=43.5/96\approx0.4531$.
\end{proof}

\emph{Remark.} This is \emph{nominal} exposure. Realized exposure
$\frac{1}{T}\sum_t\sum_im_i(t)/|I_m|$ can be strictly smaller because
$k_m(t)=\min\{\lfloor\kappa_t|I_m|\rfloor,|\mathcal{E}_m|\}$, so
Equation~\ref{eq:exposure} is an upper bound on the fraction of token-steps
receiving privileged supervision. 

\section{Coverage Annealing Does Not Renormalize Survivors}
\label{app:norenorm}

\begin{proposition}
\label{prop:norenorm}
Holding the rollout, model parameters, teacher targets, detached selection mask,
and token weights
Under Equation~\ref{eq:objective}, the auxiliary gradient contributed by a
retained token $i$ is $(\lambda/N_{\mathrm{valid}})w_i\nabla_\theta
\ell^{\mathrm{sKL}}_i$, independent of the budget and of which other tokens are
retained. Under coefficient annealing $\lambda\to\lambda_t$ at full coverage the
contribution is $(\lambda_t/N_{\mathrm{valid}})w_i\nabla_\theta
\ell^{\mathrm{sKL}}_i$, decaying at every position simultaneously. Under
selected-set normalization $1/N_{\mathrm{sel}}(t)$ with
$N_{\mathrm{sel}}(t)=\sum_mk_m(t)\ge1$, it is
$(\lambda/N_{\mathrm{sel}}(t))w_i\nabla_\theta\ell^{\mathrm{sKL}}_i$, which grows
as the budget shrinks.
\end{proposition}

\begin{proof}
Differentiating Equation~\ref{eq:objective} with $m_i(t)$, $w_i$ and teacher
targets detached, the auxiliary gradient is
$(\lambda/N_{\mathrm{valid}})\sum_{i\in I}m_i(t)w_i\nabla_\theta
\ell^{\mathrm{sKL}}_i$. The coefficient of a retained term is
$\lambda/N_{\mathrm{valid}}$, and $N_{\mathrm{valid}}=|I|$ is fixed by the
rollout, independent of $\kappa_t$, $k_m(t)$ and the mask; $\lambda$ is constant.
Replacing $\lambda$ by $\lambda_t$ scales all terms by the common factor
$\lambda_t/\lambda$, so no position retains full strength once
$\lambda_t<\lambda$. Replacing $N_{\mathrm{valid}}$ by $N_{\mathrm{sel}}(t)$
gives per-token weight $\lambda/N_{\mathrm{sel}}(t)$; since
$k_m(t)\le\lfloor\kappa_t|I_m|\rfloor$ is nondecreasing in $\kappa_t$,
$N_{\mathrm{sel}}(t)$ is nondecreasing in the budget and the weight is
nonincreasing in it, strictly whenever $N_{\mathrm{sel}}$ strictly decreases.
\end{proof}

\emph{Remark.} This concerns per-token coefficients, not the aggregate auxiliary
gradient norm, which need not decrease monotonically under coverage annealing
because the retained set concentrates on the largest $b_i$.

\section{Baselines}
\label{app:baselines}

We compare GIFT with the initial model (\textsc{Base}) and eight training
baselines, grouped by supervision source and how that supervision enters
optimization.

\noindent\textbf{Initial model and outcome-based optimization.}
These baselines establish performance before additional training and under
verifier-based reinforcement learning without teacher feedback.
\begin{itemize}[leftmargin=*,itemsep=0pt,topsep=0pt]
    \item \textbf{\textsc{BASE}:} The initial model checkpoint without the additional task-specific training considered in our experiments.

    \item \textbf{GRPO~\citep{grpo}:} Performs pure reinforcement learning using verifier-derived group-relative advantages, without teacher-derived token-level supervision.
\end{itemize}

\noindent\textbf{Distributional self-distillation.}
These methods provide dense supervision by matching next-token distributions
from a self-teacher conditioned on additional reference information or
environment feedback.
\begin{itemize}[leftmargin=*,itemsep=0pt,topsep=0pt]
    \item \textbf{OPSD~\citep{opsd}:} Conditions a self-teacher on a correct reference answer or reasoning trace and matches its next-token distributions along student-generated trajectories. Its original formulation uses a forward-KL distillation objective without verifier-derived advantages.

    \item \textbf{SDPO~\citep{sdpo}:} Conditions a self-teacher on environment feedback or verified successful solutions and distills its next-token distributions along existing student trajectories.
\end{itemize}

\noindent\textbf{Verifier-guided token-level credit assignment.}
These methods combine verifiable outcomes with scalar teacher-derived signals
by modulating token-level advantages or gating token-level rewards.
\begin{itemize}[leftmargin=*,itemsep=0pt,topsep=0pt]
    \item \textbf{RLSD~\citep{rlsd}:} Uses sampled-token probability ratios between a privileged self-teacher and the student to construct positive, clipped weights for verifier-derived advantages. This sign-preserving modulation redistributes credit across tokens and is annealed during training.

    \item \textbf{RLCSD~\citep{pan2026rlcsd}:} Contrasts sampled-token likelihoods under correct- and incorrect-solution-conditioned teachers and applies bounded, threshold-gated modulation to verifier-derived advantages while preserving their signs.

    \item \textbf{OPDVR~\citep{opdvr}:} Gates sampled-token teacher--student log-ratio rewards according to trajectory correctness, retaining non-negative rewards for correct trajectories and non-positive rewards for incorrect trajectories.
\end{itemize}

\noindent\textbf{Demonstration-guided methods.}
These methods incorporate reference reasoning trajectories into training
samples or rollout contexts rather than relying only on teacher distributions
along unassisted student rollouts.
\begin{itemize}[leftmargin=*,itemsep=0pt,topsep=0pt]
    \item \textbf{LUFFY~\citep{yan2025learningreasonoffpolicyguidance}:} Combines off-policy teacher demonstrations with on-policy student rollouts and uses regularized importance sampling to learn from demonstrations while retaining student exploration.

    \item \textbf{UFT~\citep{UFT}:} Initializes rollouts with prefixes of reference solutions, applies supervised learning to the provided prefixes, and performs reinforcement learning on the generated continuations. The expected prefix length is progressively reduced until training becomes hint-free.
\end{itemize}

These comparisons separate the source of teacher guidance from how it enters
optimization. In contrast, GIFT uses gradient-informed token selection and
coverage annealing to control the contribution of an auxiliary distributional
loss while retaining the GRPO objective throughout training.

\section{Detailed Experimental Setup}
\label{app:gift-experimental-details}

This appendix describes the training data, evaluation benchmarks, GIFT
configuration, and evaluation protocol. The model-specific optimization
and systems settings reported below refer to the OLMo-3-7B-Think
configuration. They are listed separately from the settings shared by
the experimental design.

\subsection{Training Datasets}
\label{app:gift-training-data}

\paragraph{DeepMath-103K.}
DeepMath-103K~\citep{he2025deepmath103klargescalechallengingdecontaminated} contains 103,022 English mathematical-reasoning examples covering six major topics. Each problem is paired with a final answer, a GPT-4o-estimated difficulty score, a topic label, and three DeepSeek-R1 reasoning traces. The released dataset undergoes semantic decontamination, question standardization, and answer-consistency filtering, supporting both outcome verification and reference-conditioned distillation.

\paragraph{Knights and Knaves.}
We adopt the synthetic Knights and Knaves (K\&K) task from Logic-RL~\citep{xie2025logicrlunleashingllmreasoning}, where each character is either a truth-telling knight or a lying knave. Given their logically composed statements, the model must infer the unique role assignment, which can be verified using deterministic rules. We train on puzzles with 4--8 characters and evaluate role-count generalization on larger instances of the same task.

\subsection{Evaluation Benchmarks}
\label{app:gift-benchmarks}

\paragraph{AMC 2023.}
AMC evaluates mathematical problem solving across algebra, geometry, number
theory, counting, and probability~\citep{mathai_amc23}. Our evaluation set contains
83 problems, denoted as AMC23 in the result tables, and is scored by matching
the generated final answer.

\paragraph{AIME24 and AIME25.}
AIME24 and AIME25 each contain 30 problems from the corresponding year's
AIME I and AIME II examinations~\citep{aime24,aime25}. Each answer is an
integer from 0 to 999 and is evaluated using final-answer matching. Results
for the two years are reported separately.

\paragraph{In-distribution and out-of-distribution logic.}
For logical reasoning, we train on Knights-and-Knaves puzzles
(K\&K)~\citep{xie2025logicrlunleashingllmreasoning} with 4--8 roles.
We evaluate on 500 in-distribution problems and three out-of-distribution
splits with 9, 10, and 11 roles, each containing 100 problems, to assess
generalization beyond the training range.

\subsection{GIFT Configuration}
\label{app:gift-training-configuration}

\paragraph{Privileged self-teacher.}
The student generates responses without access to the reference solution.
A self-teacher then re-evaluates the same response prefixes with the reference
solution and answer included in its context. The teacher shares the current
actor parameters, with gradients stopped through the teacher branch, and is
used only during training.

\paragraph{Sparse forward-KL distillation.}
Let $p_i$ and $q_i$ be the student and teacher distributions at position $i$,
and let $S_i$ contain the teacher's top-100 tokens. We aggregate the remaining
vocabulary into a single tail and compute
\begin{equation}
\ell_i^{\mathrm{sKL}}
=
\sum_{v\in S_i}q_i(v)\log\frac{q_i(v)}{p_i(v)}
+
q_{i,\mathrm{tail}}
\log\frac{q_{i,\mathrm{tail}}}{p_{i,\mathrm{tail}}},
\label{eq:gift-app-sparse-kl}
\end{equation}
where $p_{i,\mathrm{tail}}=1-\sum_{v\in S_i}p_i(v)$ and likewise for
$q_{i,\mathrm{tail}}$. Teacher probabilities are treated as fixed targets.

\paragraph{Token ranking and allocation.}
For eligible response tokens, we compute the gradient geometry statistics $a_i$ and $b_i$ defined in the main text. 
The ranking score is $a_i$ when $a_i < 0$ and $b_i$ otherwise; tokens without a usable GRPO direction ($i \notin D$) are \emph{unclassified} and receive score zero. Thus, conflicting tokens remain eligible but are withdrawn first as coverage decreases: when the quota exceeds the number of non-conflicting and unclassified eligible tokens, the least
opposing conflicting tokens fill the remaining capacity, as guaranteed by Theorem~\ref{thm:selection}. The quota in Equation~25 is exact---$|A_{t,m}| = K_{t,m}$ whenever $K_{t,m} \le |E_m|$---so realized coverage tracks
$\kappa_t$ deterministically at the sub-batch level.
Selection is performed independently within each training sub-batch $m$:
\begin{equation}
K_{t,m}=\left\lfloor\kappa_t|I_m|\right\rfloor,
\qquad
|\mathcal A_{t,m}|=\min\{K_{t,m},|\mathcal E_m|\},
\label{eq:gift-app-budget}
\end{equation}
where $I_m$ and $\mathcal E_m$ denote its valid and eligible response tokens,
respectively. The highest-scoring eligible tokens are selected.

\paragraph{Coverage schedule and objective.}
Training spans 96 actor steps with
\begin{equation}
\kappa_t=
\begin{cases}
1, & 1\le t\le10,\\
\dfrac{78-t}{68}, & 11\le t\le77,\\
0, & 78\le t\le96.
\end{cases}
\label{eq:gift-app-coverage}
\end{equation}
We fix the distillation coefficient to $\lambda=1$. Selected tokens receive
the complete sparse forward-KL loss, normalized by the number of all valid
response tokens, while GRPO remains active throughout training. From step 78,
teacher evaluation is skipped and optimization reduces to GRPO alone.

\subsection{Ablation Coverage and Qualitative Diagnostics}
\label{app:gift-ablations}

We examine the token-selection criterion, allocation scope, and teacher-withdrawal strategy.
Let $r_i$ and $d_i$ denote the GRPO and OPSD loss gradients with respect to the student logits.
The signed parallel component $a_i$ and orthogonal residual norm $b_i$ follow the definitions in the method section.

\paragraph{Token-selection strategies.}
\begin{itemize}[leftmargin=*,itemsep=0pt,topsep=0pt]
    \item \textbf{All Tokens:}
    Adds the OPSD loss directly to the GRPO objective at all valid response positions, without token selection. This variant evaluates joint training with unfiltered distillation.

    \item \textbf{Random Tokens:}
    Selects token positions at random under the distillation budget, without using uncertainty or gradient geometry. It tests whether selective coverage alone is sufficient.

    \item \textbf{Entropy--KL Selection:}
Following the Q3 criterion in TIP~\citep{tip}, this variant
prioritizes positions where the student is confident but disagrees
with the teacher. Let
\begin{equation}
    h_i
    =
    -\frac{1}{\log |V|}
    \sum_{v\in V}p_i(v)\log p_i(v)
\end{equation}
denote the normalized student entropy. The selection score is
\begin{equation}
    s_i^{\mathrm{Q3}}
    =
    D_{\mathrm{KL}}\left(q_i\,\|\,p_i\right)
    \left(1-h_i\right).
\end{equation}
Tokens with larger scores have greater teacher--student divergence
and lower student uncertainty. They are selected under the same
coverage budget as GIFT.

    \item \textbf{Gradient-cosine Selection:}
    Uses the cosine similarity between the OPSD and GRPO logit gradients,
    $c_i=\langle d_i,r_i\rangle/(\|d_i\|_2\|r_i\|_2)$,
    at positions with usable gradients. This criterion measures directional agreement without separately accounting for the magnitude of the orthogonal residual.

    \item \textbf{Conflict-only:}
    Uses only the conflict information provided by the parallel component $a_i$ for token selection, omitting residual-based prioritization. It isolates the contribution of conflict assessment.

    \item \textbf{Residual-only:}
    Selects tokens solely according to the orthogonal residual norm $b_i$, without distinguishing conflicting from non-conflicting positions through the sign of $a_i$.

    \item \textbf{Teacher-norm:}
    Selects tokens according to $\|d_i\|_2$, the magnitude of the teacher-induced distillation gradient with respect to the student logits. It does not separate parallel and orthogonal contributions.

    \item \textbf{Conflict-Free:}
    Excludes every token with $a_i<0$ from distillation, regardless of the available budget. Unlike GIFT's conflict-sensitive ordering, this variant never admits conflicting tokens to fill unused capacity.
\end{itemize}

\paragraph{Allocation scopes.}
These variants change the set of tokens that compete for the distillation budget.
\begin{itemize}[leftmargin=*,itemsep=0pt,topsep=0pt]
    \item \textbf{Per-rollout:}
    Ranks and selects tokens independently within each generated response. Tokens from different responses do not compete for a shared quota.

    \item \textbf{Per-prompt:}
    Ranks tokens jointly across the group of responses generated for the same prompt. Tokens can compete across responses within the group, but not across different prompts.

    \item \textbf{Training sub-batch (GIFT):}
    Ranks tokens jointly within each training sub-batch, allowing competition across the responses and prompts it contains. Selection budgets remain independent across training sub-batches.
\end{itemize}

\paragraph{Teacher-withdrawal strategies.}
\begin{itemize}[leftmargin=*,itemsep=0pt,topsep=0pt]
    \item \textbf{Sequential OPSD$\rightarrow$GRPO:}
    Performs OPSD training first and then switches to GRPO, separating distillation and outcome-based optimization into successive stages.

    \item \textbf{Weight Annealing:}
    Gradually reduces the distillation-loss coefficient, weakening teacher supervision through its weight rather than progressively removing token positions.

    \item \textbf{Coverage Annealing (GIFT):}
    Progressively reduces the token-retention budget while keeping the distillation coefficient fixed at $\lambda=1$. Selected positions receive the full distillation loss, and training becomes pure GRPO when coverage reaches zero.
\end{itemize}

The preliminary schedule sweep evaluates linear decay with $(H,T_{\mathrm{exit}})$
equal to $(10,78)$, $(0,78)$, $(20,78)$, $(10,50)$, and $(10,90)$,
and cosine decay with $(10,78)$ and $(0,78)$.
Here, $H$ denotes the initial full-coverage period and $T_{\mathrm{exit}}$ the teacher exit step.
All-token and conflict-free variants also change coverage or eligibility; schedule comparisons do not imply matched cumulative teacher exposure or training cost.

\section{Additional Experiments}
\label{appendix: additional_experi}

We provide additional experiments on model-family generalization, teacher
withdrawal, and schedule sensitivity. These analyses examine the central
design of GIFT: dense teacher guidance is used early, progressively restricted
to higher-priority tokens, and eventually removed to recover pure
outcome-grounded optimization. We evaluate whether this strategy transfers
beyond Qwen3, outperforms alternative withdrawal mechanisms, and remains
effective under different annealing configurations.

\begin{figure}[t]
    \centering
    \includegraphics[width=0.99\linewidth]{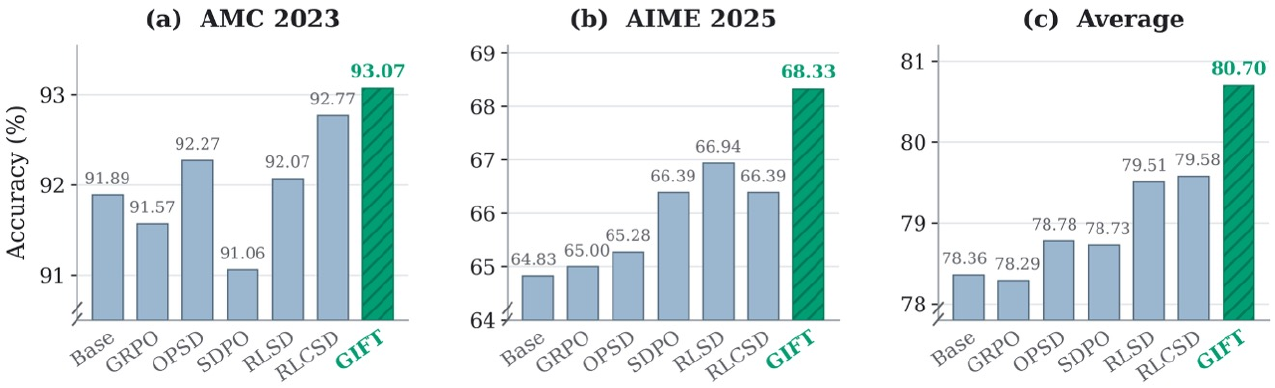}
    \vspace{-3pt}
    \caption{Results on OLMo3-7B, reported as mean@12.
Avg.\ is the equally weighted average over the two benchmarks.
Green numbers indicate absolute gains over Base.}
    \vspace{-8pt}
    \label{fig:gift-olmo}
\end{figure}

\paragraph{Evaluation beyond the Qwen3 family.}
On OLMo-3-7B-Think, GIFT achieves the highest mean@12 on both AMC23 and AIME25
(Figure~\ref{fig:gift-olmo}). Its two-benchmark average of 80.70\% exceeds
Base and the strongest baseline, RLCSD, by 2.34 and 1.12 percentage points,
respectively. The gain over Base is larger on AIME25 than on AMC23
(3.50 versus 1.18 points), consistent with the trend observed across the
Qwen3 models. These improvements are consistent with GIFT's central
motivation: rather than uniformly applying teacher supervision, GIFT
prioritizes non-conflicting updates that contain directional information
beyond the local RL gradient, while progressively reducing their coverage.
This allows dense guidance to complement the outcome-grounded objective
during early training without remaining the final optimization target.
The consistent gains on OLMo-3-7B-Think further suggest that this mechanism
is not specific to the Qwen3 model family.

\begin{table}[t]
\centering
\caption{Teacher-withdrawal schedules on Qwen3-8B (mean@12, \%).
Avg. equally weights the three benchmarks. Bold marks the best score in each column, including ties.}
\label{tab:gift-schedule}
\begin{tabular}{@{}lcccc@{}}
\toprule
Training schedule & AMC23 & AIME24 & AIME25 & Avg. \\
\midrule
OPSD $\rightarrow$ GRPO & 89.06 & 75.28 & 67.78 & 77.37 \\
Weight annealing & {89.66} & 72.22 & 66.39 & 76.09 \\
\midrule
\rowcolor{gray!12}
GIFT & {89.66} & {76.11} & {71.11} & {78.96} \\
\bottomrule
\end{tabular}%
\end{table}

\paragraph{Coverage annealing versus alternative withdrawal strategies.}
Table~\ref{tab:gift-schedule} compares GIFT with two alternative teacher-withdrawal
strategies. In sequential OPSD$\rightarrow$GRPO training, the model is first
optimized using only the OPSD objective and then switched to pure GRPO, so
dense teacher supervision and outcome-based optimization are applied in
separate stages. Weight annealing instead combines OPSD with GRPO throughout
the transition but applies distillation to all valid tokens and gradually
decays its global loss coefficient, without annealing token coverage. GIFT
outperforms these two alternatives by 1.59 and 2.87 percentage points on
average, respectively. Unlike sequential training, GIFT keeps the
outcome-grounded GRPO objective active while teacher guidance is available.
Unlike weight annealing, it progressively removes lower-priority positions
while preserving the full distillation strength at selected tokens. This
design retains concentrated teacher guidance where it is more compatible
with and complementary to the RL update, rather than uniformly weakening
every teacher contribution. Once coverage reaches zero, the auxiliary loss
vanishes and optimization continues with pure GRPO. The results support joint
optimization with selective coverage annealing over either stage-wise
training or uniform coefficient decay.

\begin{table}[t]
\centering
\caption{Preliminary sensitivity to the full-coverage period $H$,
teacher exit step $T_{\mathrm{exit}}$, and decay shape on Qwen3-8B (mean@12, \%).
Avg. equally weights the three benchmarks. Bold marks the best score in each evaluation column.}
\label{tab:gift-sensitivity}
\begin{tabular}{@{}lcccccc@{}}
\toprule
\multirow{2}{*}{Schedule}
& \multicolumn{2}{c}{Hyperparameters}
& \multicolumn{4}{c}{Evaluation} \\
\cmidrule(lr){2-3}\cmidrule(lr){4-7}
& $H$ & $T_{\mathrm{exit}}$ & AMC23 & AIME24 & AIME25 & Avg. \\
\midrule
\rowcolor{gray!12}
Linear (GIFT) & 10 & 78 & 89.66 & 76.11 & \textbf{71.11} & \textbf{78.96} \\
Linear & 0  & 78 & 89.36 & 75.56 & 69.17 & 78.03 \\
Linear & 20 & 78 & 89.56 & 75.28 & 66.94 & 77.26 \\
Linear & 10 & 50 & 90.16 & \textbf{76.39} & 65.83 & 77.46 \\
Linear & 10 & 90 & 89.86 & 75.00 & 68.06 & 77.64 \\
\midrule
Cosine & 10 & 78 & \textbf{90.46} & 76.11 & 67.22 & 77.93 \\
Cosine & 0  & 78 & 89.96 & 73.89 & 67.22 & 77.02 \\
\bottomrule
\end{tabular}%
\end{table}

\paragraph{Sensitivity to warm-up and exit steps.}
Table~\ref{tab:gift-sensitivity} presents a preliminary sweep of the full-coverage
warm-up $H$ and teacher exit step $T_{\mathrm{exit}}$.
With $T_{\mathrm{exit}}=78$, removing warm-up or extending it to 20 steps reduces
the average from 78.96\% to 78.03\% or 77.26\%, respectively.
With $H=10$, moving teacher exit to step 50 or 90 gives 77.46\% or 77.64\%.
The default configuration therefore performs best among the tested linear schedules.
These results favor an initial full-coverage period followed by gradual withdrawal,
while showing that either shortening or extending teacher guidance can reduce performance.

Together, these experiments support GIFT's combination of conflict-aware token selection,
shared training-sub-batch allocation, and progressive teacher withdrawal.
The evidence concerns reasoning performance and observed selection behavior,
rather than direct measurements of parameter-space gradient compatibility or optimization plasticity.

\section{Qualitative Analysis}
\label{appendix:qualitative}

We provide three qualitative analyses to examine both the reasoning outcomes
and token-selection behavior of GIFT. First, we compare representative
solutions from GIFT, GRPO, and OPSD to illustrate how combining dense guidance
with outcome-grounded optimization affects reasoning accuracy. Second, we
visualize the token positions retained at different training steps, showing
how supervision becomes progressively sparse within a response. Third, we
analyze retention in the space of the signed parallel component $a_i$ and
orthogonal residual norm $b_i$, connecting the observed selection pattern to
GIFT's piecewise ranking rule. Together, these visualizations illustrate how
GIFT transitions from broad teacher supervision to concentrated guidance on
higher-priority positions as the coverage budget contracts.

\begin{figure}[t]
    \centering
    \includegraphics[width=0.98\linewidth]{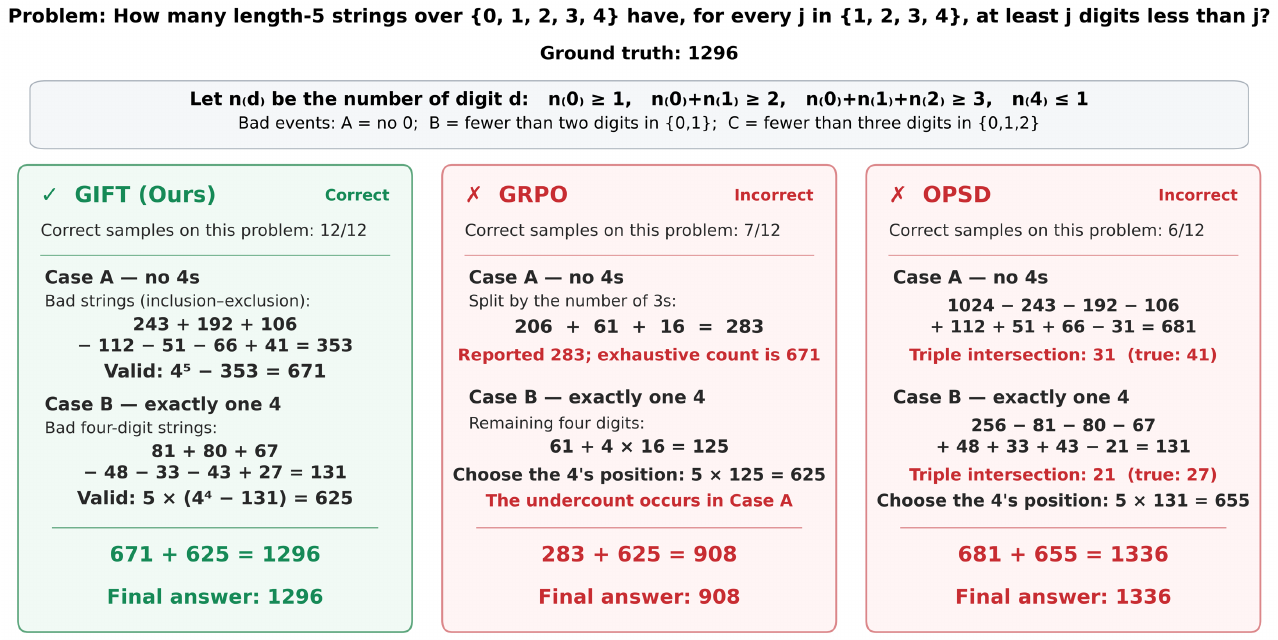}
    \vspace{-3pt}
\caption{
Example solutions to a combinatorial counting problem.
The displayed GIFT solution is correct, while the GRPO and OPSD
solutions contain counting errors.
Correct-sample counts summarize 12 generations per method.
}
    \vspace{-8pt}
\label{fig:gift-counting-case}
\end{figure}

\paragraph{Qualitative analysis of reasoning accuracy.}
Figure~\ref{fig:gift-counting-case} compares solutions to a combinatorial
counting problem. All three displayed solutions separate strings
according to whether they contain a digit 4, but differ in how
accurately they execute the subsequent counting.
GRPO undercounts the no-4 case as 283 instead of 671, while OPSD
miscalculates intersection terms in inclusion--exclusion and obtains
an incorrect total of 1336. GIFT correctly handles both cases,
yielding $671+625=1296$.
Across 12 sampled responses to this problem, GIFT produces 12 correct
answers, compared with 7 for GRPO and 6 for OPSD.
The contrast connects to our introductory motivation: a plausible
reasoning structure does not ensure correct intermediate decisions,
and neither outcome-based training nor self-distillation alone
eliminates such errors in these examples.
GIFT's more consistent answers provide a qualitative illustration
of the goal of combining dense guidance with outcome-grounded
optimization, complementing the quantitative ablations of its
selection and withdrawal mechanisms.

\begin{figure}[t]
    \centering
    \includegraphics[width=0.98\linewidth]{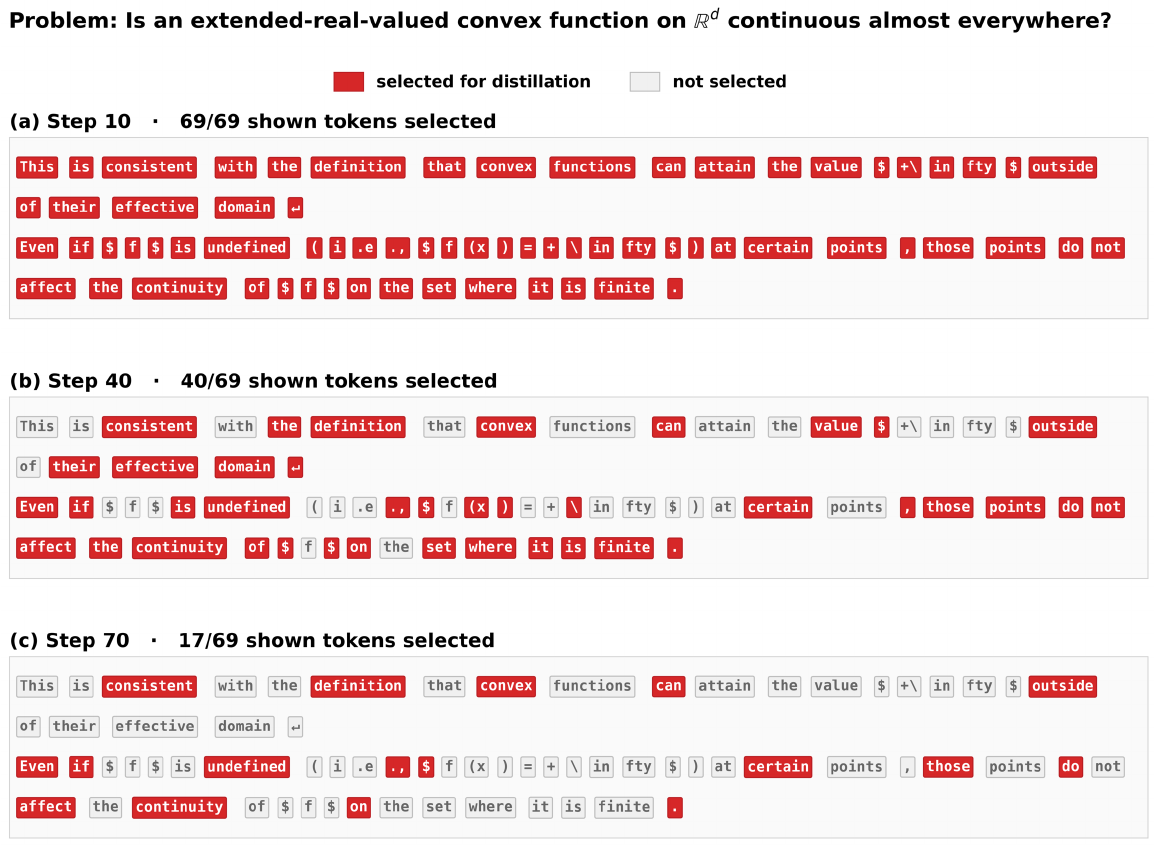}
    \vspace{-3pt}
    \caption{
    Token-level distillation selection in a reasoning excerpt.
    Red marks selected tokens and gray marks unselected tokens.
    The number of selected positions decreases from 69 to 40 to 17
    across the displayed steps; counts refer only to the shown excerpt.
    }
    \vspace{-8pt}
    \label{fig:gift-token-selection-case}
\end{figure}

\paragraph{Qualitative analysis of token selection.}
Figure~\ref{fig:gift-token-selection-case} illustrates how distillation
coverage changes within a reasoning excerpt about convex functions.
All 69 displayed tokens are selected at step 10, compared with 40
at step 40 and 17 at step 70.
The remaining positions include words such as ``convex,'' ``outside,''
and ``continuity,'' while many surrounding words and mathematical
symbols no longer receive distillation.
Supervision therefore becomes sparse within sentences rather than
being removed at the sentence level.
This example illustrates the allocation problem raised in the
introduction: reducing teacher influence involves deciding where
guidance remains, not merely weakening every teacher contribution.
GIFT makes this decision through gradient-based ranking within each
training sub-batch. Selected positions retain the full distillation
loss, while all valid tokens remain included in the GRPO objective.
The visualization shows where supervision is applied, rather than
establishing the semantic importance or correctness of individual tokens.

\begin{figure}[t]
\centering
\includegraphics[width=0.99\linewidth]{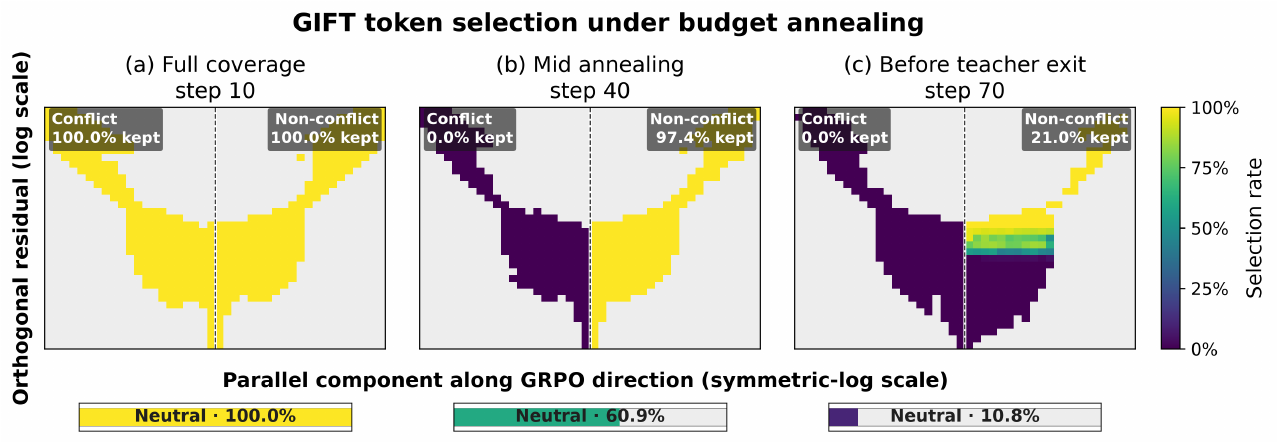}
\caption{Token retention under coverage annealing.
The horizontal and vertical axes show the signed parallel component $a_i$
and orthogonal residual norm $b_i$ of the distillation gradient relative to the GRPO direction.
Colors indicate the fraction of eligible tokens retained within each bin.
Annotations summarize retention for conflicting ($a_i<0$) and non-conflicting ($a_i\geq0$)
tokens with usable GRPO reference directions.
Bottom bars separately report retention for neutral tokens without such a direction.}
\label{fig:gift-annealing}
\end{figure}

\paragraph{Token selection under budget annealing.}
Figure~\ref{fig:gift-annealing} shows how GIFT concentrates distillation
as the token budget contracts. At step 10, all eligible tokens are
retained, including conflicting and neutral positions.
At step 40, conflicting-token retention is zero in the displayed
statistics, while 97.4\% of non-conflicting tokens and 60.9\% of
neutral tokens remain selected. By step 70, these rates decrease to
21.0\% and 10.8\%, respectively, with non-conflicting selection
concentrated at larger orthogonal residuals.
This pattern follows the piecewise ranking: conflicting tokens receive
negative scores, non-conflicting tokens are prioritized by residual
magnitude, and neutral tokens remain eligible with zero scores.
Competition within each training sub-batch therefore produces a
selective withdrawal of supervision rather than uniform attenuation.
Retained tokens receive the full, unprojected distillation loss,
while all valid response tokens remain included in the GRPO objective.

\end{document}